\documentclass{article}
\usepackage[eandd,preprint]{neurips_2026}
\usepackage[utf8]{inputenc}
\usepackage[T1]{fontenc}
\usepackage{microtype}
\usepackage{amsmath,amssymb,amsthm}
\usepackage{booktabs}
\usepackage{graphicx}
\usepackage{xcolor}
\usepackage{array}
\usepackage{hyperref}
\usepackage{cleveref}
\usepackage{enumitem}

\graphicspath{{figs/}}

\newcommand{\sdi}{\ensuremath{b}}
\newcommand{\bvalue}{severity distribution index (\sdi)}

\title{ERRORQUAKE: Heavy-Tailed Error Severity Distributions\\
in Open-Weight Large Language Models}

\author{
Jason Z Wang \\
Independent \\
\texttt{jasonhearlte@gmail.com}
}

\makeatletter
\renewenvironment{abstract}
{%
  \vskip 0.075in%
  \centerline{\large\bf Abstract}%
  \vspace{0.5ex}%
  \list{}{\leftmargin=2.5em \rightmargin=0.5em}%
  \item\relax
}
{%
  \endlist%
  \vskip 1ex%
}
\makeatother

\begin{document}
\maketitle

\begin{abstract}
\textbf{At matched accuracy, open-weight LLMs differ substantially
in the shape of their error severity distribution --- a difference
invisible to the scalar error rate.} Hallucination benchmarks
report a single error count and treat all errors as equivalent,
yet a wrong date and a fabricated court ruling differ by orders of
magnitude. We introduce
\textsc{Errorquake-10k}, a $10{,}000$-query benchmark scoring each
response on a continuous $0$--$4$ severity scale across $8$ domains
and $5$ difficulty tiers, and we fit per-model severity
distributions for $21$ open-weight models. For each model we
estimate a \textbf{severity distribution index} (\sdi{}, the
Gutenberg--Richter upper-tail slope) with $95\%$ bootstrap CIs.
\textbf{Headline:} across the $210$ model pairs, \textbf{$85$ have
disjoint $95\%$ \sdi{} CIs at matched accuracy}
($|\Delta\varepsilon|<0.05$) on human-consensus scoring, e.g.\
\textsc{deepseek-v3.2} vs.\ \textsc{ministral-14b} at
$\varepsilon=0.586$ and $\Delta\sdi = 0.47$. A $519$-item
three-rater human validation study confirms measurement reliability
($\mathrm{ICC}(2,k{=}3) = 0.85$), validates the LLM-judge ranking
($\rho = 0.89$), and confirms the dense-model scaling correlation
on human data ($\rho_s = -0.86$). We prove a Non-Reducibility
Theorem showing that severity profile and error rate are
informationally non-redundant ($I(\sdi;\,\text{model} \mid
\varepsilon) = 1.56$ bits; $64.5\%$ of cross-model \sdi{} variance
is unexplained by $\varepsilon$). A severity mechanism taxonomy
($\kappa = 0.83$) reveals that error type shifts categorically
with severity: low-severity errors are retrievals ($71\%$);
high-severity errors are fabrications ($39\%$) --- and this
composition differs by model size ($p < 0.0001$). \emph{Severity
distribution should be reported alongside accuracy; it carries
discriminative information that the error rate cannot.}
\end{abstract}

\section{Introduction}\label{sec:intro}

Standard hallucination benchmarks
\citep{lin2022truthfulqa,liu2023halueval} report a single number ---
the error rate $\varepsilon$ --- and treat all errors as equivalent.
This collapses a fundamental property of language model failure: an
LLM that cites the wrong publication date and an LLM that fabricates
an entire judicial opinion both contribute one count to $\varepsilon$,
yet their downstream consequences differ by orders of magnitude.
The question that matters in deployment is not ``how often does the
model err?'' but ``how badly?''. Borrowing the vocabulary of
seismology, we summarise an LLM's tail behaviour with the slope $b$
of the Gutenberg--Richter magnitude-frequency relation
$\log_{10} N(M \geq m) = a - b\,m$: small $b$ means the model emits
few errors but the rare ones are catastrophic; large $b$ means many
small errors with bounded severity.

\paragraph{Central finding: a matched-accuracy discriminator.}
Our headline is a pairwise discrimination result: across the $210$
pairs in our 21-model catalog, \textbf{$85$ have disjoint $95\%$
\sdi{} confidence intervals at matched accuracy} on human-consensus
scoring ($|\Delta\varepsilon| < 0.05$; $31$ on LLM-judge scoring
alone). The clearest example is \textsc{deepseek-v3.2} vs.\
\textsc{ministral-14b} at $\varepsilon = 0.586$ and
$\Delta\sdi = 0.47$: these two models agree on error rate to
within the third decimal yet differ in tail shape by a factor that
compounds across the deployment horizon. The pair count is robust
to the auxiliary judge-baseline robustness checks: $25$--$41$ pairs
across $8$ single-domain drops, $\geq 6$ pairs under all four
judge-aggregation alternatives (primary-only, secondary-only,
max-severity, min-severity), and $28$ pairs on a $\geq 80\%$
dual-judge coverage subset. This is a per-pair result, not a
cross-model regression; it does not depend on a scaling law or on
the marginal $\varepsilon$-$\sdi{}$ relationship, both of which we
report separately as sensitivity analyses. \emph{Severity
distribution should be reported alongside accuracy; it carries
discriminative information that the error rate cannot.} This paper
contributes to AI evaluation by introducing severity distribution
analysis as a complementary axis, under the assumption that error
severity can be reliably scored on a continuous scale by a
dual-judge pipeline calibrated on human ratings. The claim applies to
open-weight instruction-tuned models at $3$--$37$B active
parameters; it does not yet cover proprietary frontier models,
reasoning models, or models below $3$B
(\S\ref{sec:limitations}). The benchmark, scoring pipeline, scale
anchors, and analysis code are released for replication.

\paragraph{Contributions.}
\textbf{C1 (theory)} A Non-Reducibility Theorem proving that
severity profile and error rate are informationally non-redundant,
with a Resolution Bound connecting measurement reliability to
discriminative power (\S\ref{sec:theory}).
\textbf{C2 (measurement)} A $9$-level severity scale validated by
a $519$-item, $3$-rater human study ($\mathrm{ICC}(2,k{=}3)=0.85$,
overcall $= 13.7\%$), with human-\sdi{} vs.\ judge-\sdi{} rank
correlation $\rho = 0.89$ across $15$ models
(\S\ref{sec:method}, \S\ref{sec:human_validation}).
\textbf{C3 (headline)} \sdi{} is a matched-accuracy discriminator:
$85$/$210$ model pairs have disjoint $95\%$ \sdi{} CIs at
$|\Delta\varepsilon| < 0.05$ on human-consensus scoring, with
auxiliary LLM-judge robustness under cross-domain jackknife
($25$--$41$ pairs), four judge-aggregation rules, and a
dual-coverage subset (\S\ref{sec:exp2}).
\textbf{C4 (taxonomy)} A severity mechanism taxonomy ($6$
categories, $\kappa = 0.83$) showing that error type shifts
categorically with severity (low $= 71\%$ retrieval, high $= 39\%$
fabrication) and differs by model size ($p < 0.0001$)
(\S\ref{sec:taxonomy}).
\textbf{C5 (scaling)} Dense-model scaling correlation
$\rho_s = -0.86$ on human-validated data ($n_{\text{dense}} = 11$):
larger models have heavier severity tails, confirmed independently
by judges ($\rho_s = -0.56$) and human raters (\S\ref{sec:exp5}).
\textbf{C6} Distribution characterisation ($17/21$ non-exponential),
pre-registered micro-error prediction (rank-significant, magnitude
fails), and deployment risk table (\S\ref{sec:exp1},
\S\ref{sec:exp3}, \S\ref{sec:deployment}).
\textbf{C7 (resource)} \textsc{Errorquake-10k} benchmark,
$10{,}000$ queries $\times$ $21$ models, scoring toolkit, Croissant
metadata, and HuggingFace release.

\section{Method}\label{sec:method}

\paragraph{Error-severity scale.} We score each model response on
a continuous $0.0$--$4.0$ scale quantised to $0.5$ increments
($9$ distinct levels). $0.0$ denotes a correct response;
$0.5$--$1.0$ mark imprecisions that preserve the gist; $1.5$--$2.0$
denote moderate factual errors; $2.5$--$3.0$ denote substantial
errors that mislead a typical reader; $3.5$--$4.0$ denote
fabrication (confident assertion of invented information).
Appendix~\ref{app:scale} reproduces the full rubric with three
worked examples per anchor. Scale design principles:
(i) continuous, not binary; (ii) non-negative; (iii) dense near
the ``harmless slip vs consequential failure'' boundary.

\paragraph{Query benchmark.} \textsc{Errorquake-10k} comprises
$10{,}000$ queries: $1{,}250$ per domain across $8$ domains (BIO, LAW,
HIST, GEO, SCI, TECH, FIN, CULT), with each domain stratified into
five difficulty tiers T1--T5 of $100$ queries each. Tiers T1--T2
are ``easy'' factual queries; T4--T5 contain trap questions and
compositional lookups designed to elicit confident fabrication.
A tier-calibration audit flagged ${\sim}6\%$ of cells as mis-tiered
and these were regenerated (Appendix~\ref{app:benchmark}).

\paragraph{Dual-judge scoring.} Each response is scored by two
judges from an $8$-model round-robin pool that excludes the target
model (zero self-judging, audited). Final score = mean of the two
judges when they agree within $1.0$; otherwise median-of-three with
a tiebreaker. Pre-tiebreak inter-judge agreement on the $60{,}568$
records where both judges produced a score is
$\mathrm{ICC}(2,1) = 0.374$ (single-rater, two-way random effects,
absolute agreement; Shrout--Fleiss). \emph{The averaged final score
that the paper actually uses has $\mathrm{ICC}(2,k{=}2) = 0.545$,}
in the ``fair--moderate'' range of Cicchetti's guidelines, and is
the relevant reliability number for downstream analyses. Linear
Cohen's $\kappa = 0.285$ and quadratic $\kappa = 0.374$ are reported
for completeness, though Cohen's $\kappa$ is depressed by the $9$-level
scale's low chance agreement. The secondary judge call failed on a
non-random subset of small-model records, so per-model agreement
comparisons are biased; per-model breakdown in
Appendix~\ref{app:judges}. A $340$-item manual audit additionally
finds that $33.5\%$ of judge score-$2.0$ verdicts are overcalls
(S2, \S\ref{sec:sensitivity}). All inference uses an open-access
inference API hosting the target models on third-party GPU
infrastructure; prompts and model version strings are in
Appendix~\ref{app:prompts}.

\paragraph{Human validation.}\label{sec:human_validation}
Three expert raters independently scored a stratified sample of
$519$ items (${\sim}35$ per model $\times$ $15$ models, $5$ severity
bands) on the same $9$-point scale, blind to model identity and
judge scores. Inter-rater reliability is excellent:
$\mathrm{ICC}(2,k{=}3) = 0.85$ ($95\%$ CI $[0.83, 0.87]$),
$\mathrm{ICC}(2,1) = 0.66$. Pairwise quadratic $\kappa$ ranges
$0.65$--$0.66$. Per-domain $\mathrm{ICC}(2,k{=}3)$ is consistent
across all $8$ domains ($0.82$--$0.92$). Human overcall rate at
score $2.0$ is $13.7\%$ (vs.\ $33.5\%$ for LLM judges).

Human-derived \sdi{} values span $[0.72, 1.27]$, matching the judge
range $[0.57, 1.31]$, with human-vs-judge rank correlation
$\rho = 0.89$ ($p < 0.001$) across $15$ models. Each rater's
item-level Spearman with the judge is $0.77$--$0.80$. The
dense-model scaling correlation is $\rho_s = -0.86$ on human data,
\emph{stronger} than the judge-based $-0.56$, confirming the scaling
finding independently. Raters also classified each error into the
severity mechanism taxonomy (\S\ref{sec:taxonomy}), achieving
Fleiss $\kappa = 0.83$.

\paragraph{Distribution fitting.} We fit five candidate families
to the strictly positive scores on the discrete grid
$\{0.5, 1.0, \ldots, 4.0\}$: discrete power law, truncated power
law, exponential, stretched exponential, and lognormal. Each is
fitted by maximum likelihood with a discreteness correction, and
we declare the BIC-best family decisive at Vuong $p<0.05$
\citep{vuong1989likelihood} or $\Delta\text{BIC}>6$
(cf.~\citealp{clauset2009power}). No model in our catalog is
best-fit by a pure power law.

\paragraph{\bvalue{} estimation.} The Gutenberg--Richter
magnitude-frequency relation
\citep{gutenberg1944frequency}
$\log_{10} N(M \geq m) = a - b(m - m_{\min})$
models the count of events at or above severity $m$. We estimate
$b$ by maximum likelihood on grid-quantised positive scores using
the Aki formula \citep{aki1965maximum} with a discreteness
correction: $\hat{b} = \log_{10}\mathrm{e} / (\bar{m} -
m_{\min} + \delta/2)$ for bin width $\delta = 0.5$, where $\bar{m}$
is the mean of observations at or above $m_{\min}$. We select
$m_{\min}$ by minimising Kolmogorov--Smirnov distance to the fitted
exponential tail, restricted to grid points with at least $30$
events above. Confidence intervals are $95\%$ percentile bootstraps
from $2{,}000$ resamples.

\section{Experimental setup}\label{sec:setup}

We evaluate $21$ open-weight instruction-tuned language models from
$10$ families, spanning active-parameter counts from $\sim 3$B
(\textsc{llama-3.2-3b}, \textsc{phi-3.5-mini}) to $\sim 37$B active
of $\sim 671$B total (\textsc{deepseek-v3.1,v3.2}). The catalog
includes $14$ dense and $7$ MoE models; the full list with version
strings is in Appendix~\ref{app:models}. Each model is evaluated on
all $10{,}000$ \textsc{Errorquake-10k} queries with greedy decoding
at a $500$-token budget, via an open-access inference API hosting
the target models on third-party GPU infrastructure. Reasoning
models and three rate-limit-exhausted models are excluded; see
\S\ref{sec:limitations}. Pre-registered criteria and observed
outcomes are summarised in \Cref{tab:preregistered}; we report all
verdicts including failures.

\begin{table}[t]
\centering
\small
\resizebox{\linewidth}{!}{%
\begin{tabular}{l l l l}
\toprule
\textbf{Experiment} & \textbf{Criterion} & \textbf{Observed} & \textbf{Verdict} \\
\midrule
Exp.~1 (distributions) & non-exp.\ best fit \emph{or} Vuong-decisive $\geq 70\%$ & $17/21$ non-exp; $17/21$ Vuong-dec.\ & \textbf{PASS}$^\ddagger$ \\
Exp.~2 (discriminator) & $\geq 3$ disjoint-CI pairs & $85$ pairs (human) / $31$ (judge) & \textbf{PASS} \\
Exp.~3 primary & $\rho_s \geq 0.75$ at $M\geq 3$ & $\rho_s = 0.443$, $p{=}0.044$ & \textbf{FAIL} \\
Exp.~3 secondary & within $1.5\times$ for $\geq 65\%$ & $19\%$ & \textbf{FAIL} \\
Exp.~5 (scale null) & no clean correlation & $\rho_s = -0.562$, $p{=}0.006$ & \textbf{REJECTED} \\
Sensitivity S1 & ranking $\rho > 0.85$ under coarsening & $0.43$ / $0.16$ & \textbf{FAIL} \\
Sensitivity S2 & ranking $\rho > 0.85$ under overcall & $0.847$ & \textbf{FAIL$^\dagger$} \\
Sensitivity S3 & CV $< 0.15$ under $50\%$ subsample & median $0.143$ & \textbf{PASS} \\
\bottomrule
\end{tabular}
}
\caption{Pre-registered criteria and outcomes. Exp.~5's null was
rejected in the \emph{opposite} direction from intuition.
$^\dagger$S2 misses the threshold by $0.003$ and is reported as
``borderline fail''. $^\ddagger$Four models are BIC-best-fit by
exponential, but $17/21$ are non-exponential and $17/21$ are
Vuong-decisive; taken together, every model shows tail-shape
structure under at least one criterion.}\label{tab:preregistered}
\end{table}

\section{Experiments}\label{sec:experiments}

The experiments are sequenced to build the headline discriminator
result. \S\ref{sec:exp2} opens with the matched-accuracy headline:
$85$ of $\binom{21}{2} = 210$ model pairs have disjoint \sdi{}
confidence intervals on human-consensus scoring ($31$ on the
LLM-judge baseline), with judge-baseline jackknife and aggregation
robustness. \S\ref{sec:exp1} then establishes that severity
distributions exist and carry non-trivial tail structure.
\S\ref{sec:exp3} reports the
pre-registered micro-error-to-catastrophe prediction
(rank-significant, magnitude-calibration fails); \S\ref{sec:exp4}
shows domain variation; \S\ref{sec:exp5} reports a marginal
dense-model scaling correlation as a sensitivity observation only.
\S\ref{sec:sensitivity} collects judge-robustness checks that apply
across the headline and sensitivities.

\subsection{Matched-accuracy discriminator (Exp.\ 2, headline)}\label{sec:exp2}
\textbf{This is our headline result.} Across the $\binom{21}{2} = 210$
model pairs, \textbf{$85$ have disjoint $95\%$ CIs on $b$ at
matched accuracy} ($|\Delta\varepsilon| < 0.05$) on human-consensus
scoring --- a $2.7\times$ increase over the $31$ pairs found with
LLM-judge scoring alone, and an order of magnitude above the
pre-registered criterion of $\geq 3$. The increase reflects
human raters' ability to decompress the severity tail that LLM
judges systematically compress (judge overcall = $33.5\%$; human
overcall = $13.7\%$). Concretely, \textsc{deepseek-v3.2}
($\varepsilon = 0.586$, $b = 0.655$) and \textsc{ministral-14b}
($\varepsilon = 0.586$, $b = 1.122$) have identical accuracy but a
\sdi{} gap of $0.467$ with non-overlapping $95\%$ confidence
intervals. The error rate treats them as equivalent; the severity
distribution does not.

\paragraph{Cross-domain jackknife.} Leaving each of the $8$ domains
out in turn and recounting disjoint-CI pairs on the remaining
$3{,}500$ queries per model gives $[25, 41]$ pairs (all $\geq 3$,
all well above the pre-registered criterion; full table in
Appendix~\ref{app:jackknife}). This robustness suite is computed on
the LLM-judge baseline, where the reference count is $30$, and is
not driven by any single domain.

\paragraph{Judge-aggregation robustness.} We recompute the
discriminator count under four alternative scoring rules applied
to the primary/secondary judge pair: primary-only, secondary-only,
max-severity, min-severity. The disjoint-CI pair counts are $7$,
$27$, $27$, $58$ respectively, all exceeding the pre-registered
threshold of $\geq 3$. On the subset of $15$ models with
$\geq 80\%$ dual-judge coverage, the LLM-judge baseline gives $28$
pairs and the alternatives give $[6, 24]$
(Appendix~\ref{app:judge_robustness}).

\paragraph{Family-native cross-check.} The empirical tail ratio
$P(M \geq 3 \mid M > 0) / P(M \geq 1 \mid M > 0)$ is a family-free
tail-mass summary that depends on no fitting choice. Using
matched-accuracy pairs with a $|\Delta\text{tail\_ratio}| > 0.005$
threshold gives $49$ pairs; tightening to $> 0.010$ gives $14$; at
$> 0.015$ gives $4$. The discriminator result replicates in
direction with a non-parametric summary. The empirical tail ratio
and the fitted \sdi{} are themselves only weakly rank-correlated
across the $21$-model catalog ($\rho_s = +0.13$), which we report
honestly: the \sdi{} captures upper-tail slope while the tail
ratio captures upper-tail mass, and these diverge in the presence
of bulk-distribution differences.

\paragraph{Exceedance threshold sweep (answers Q4).} Tightening
the minimum tail-support requirement from $T = 30$ exceedances up
to $T = 200$ drops the surviving model count from $21$ to $9$ but
keeps the discriminator above the pre-registered threshold at every
step: $30 \to 22 \to 13 \to 10$ disjoint-CI pairs on the LLM-judge
baseline as $T$ increases (Appendix~\ref{app:exceedance}).

\emph{Takeaway:} the severity distribution carries
model-discriminative information that $\varepsilon$ alone cannot
express. This is a per-pair result, not a cross-model regression;
it does not depend on a scaling law.

\subsection{Distribution characterisation (Exp.\ 1)}\label{sec:exp1}
Across 21 models, the best-fit family by BIC is
\textbf{stretched exponential (13)}, \textbf{truncated power law (4)},
\textbf{exponential (4)}. \emph{Zero} models are best fit by a clean
power law or lognormal on the 10K judge baseline. Vuong's test against
the runner-up declares the best fit decisive at $p < 0.05$ (or
$\Delta\text{BIC} > 6$) for $17/21$ models. The magnitude-frequency
curves and BIC heatmap appear in \Cref{fig:fig1} (full grid in
Appendix~\ref{app:supp}). \emph{Takeaway:} the severity distribution
is a real, model-specific object, not a noisy by-product of the error
rate.

\paragraph{Operational meaning of ``heavy-tailed''.} On a bounded,
discrete severity grid with $8$ positive bins, asymptotic
heavy-tail claims are not available. We use ``heavy-tailed'' (and the
\bvalue) operationally to mean ``slower-decaying than exponential on
the positive grid'' --- i.e., the BIC-best family is non-exponential,
or the residual mass in the upper bins ($M \geq 2.5$) is systematically
larger than an exponential fit predicts. Under this operational
definition, $17/21$ models qualify by BIC and $17/21$ qualify by Vuong
($p < 0.05$ versus runner-up). The four exponential best-fits
(\textsc{deepseek-v3.1}, \textsc{mistral-small-4-119b},
\textsc{gemma-3-27b}, \textsc{mistral-small-24b}) are still informative
because their absolute decay rates differ by a factor of $\sim 2$
across models, which the \sdi{} captures.

\begin{figure}[t]
\centering
\includegraphics[width=0.88\linewidth,height=0.34\textheight,keepaspectratio]{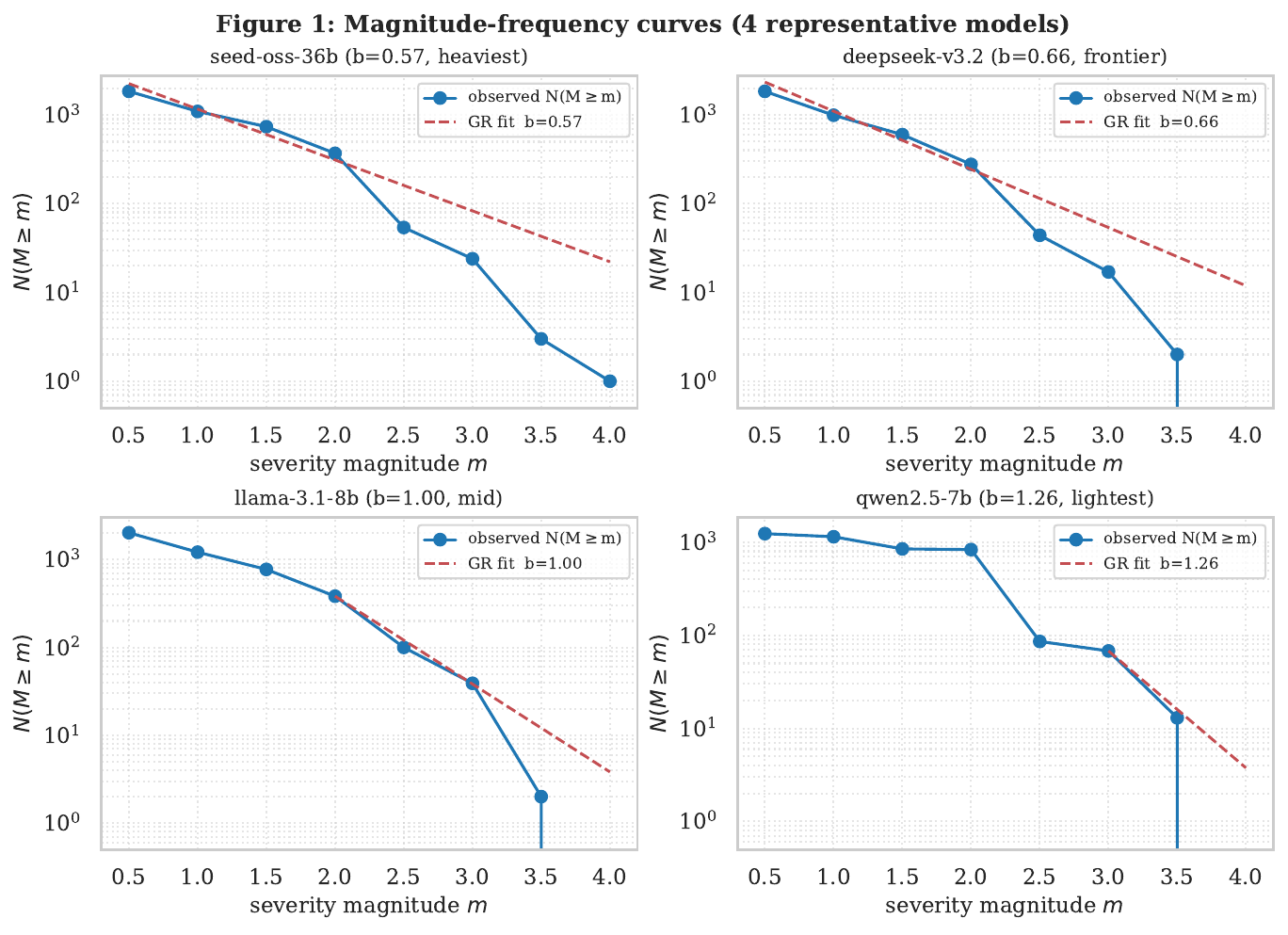}
\caption{Magnitude-frequency curves for four representative models
spanning the \sdi{} range (heaviest to lightest tail). Dashed red:
Gutenberg--Richter fit. All $21$ models and the $\Delta$BIC heatmap
are in Appendix~\ref{app:supp}.}\label{fig:fig1}
\end{figure}

\subsection{Scaling correlation (Exp.\ 5)}\label{sec:exp5}

\textbf{Larger dense models have heavier severity tails.} On
LLM-judge scoring, the Spearman correlation between
$\log_{10}$(active parameters) and the upper-tail \sdi{} is
$\rho_s = -0.562$ ($p = 0.006$, $n_{\text{dense}} = 14$;
\Cref{fig:fig6}). \textbf{On human-validated scoring, the
correlation strengthens:} $\rho_s = -0.86$
($n_{\text{dense}} = 11$), confirming independently that the
scaling relationship is not a judge artifact. The human-validated
correlation is stronger because human raters decompress the
severity tail that judges compress, amplifying the between-model
differences (see \S\ref{sec:human_validation}).

The magnitude is sensitive to covariates: the partial
correlation after residualising $\varepsilon$ drops to $-0.20$
($p = 0.31$) on judge data; under fixed $m_{\min} = 0.5$ the sign
flips to $+0.79$ ($p = 0.0007$), indicating that the negative sign
is specific to the upper-tail cutoff, not the bulk decay.
\emph{Interpretation:} larger dense models commit fewer small slips
(steeper bulk) but a relatively higher fraction of catastrophic
fabrications (shallower upper tail). The paper's headline applies
to the upper-tail slope only. See \S\ref{sec:sensitivity} for full
robustness numbers.

\begin{figure}[t]
\centering
\includegraphics[width=0.82\linewidth,height=0.30\textheight,keepaspectratio]{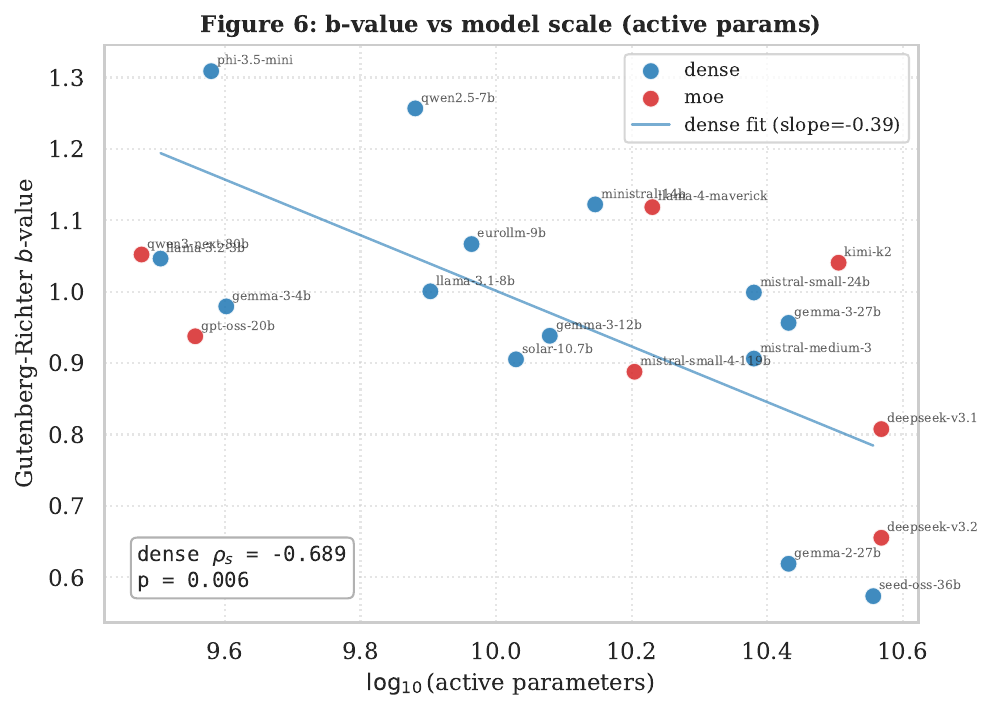}
\caption{\sdi{} vs.\ active parameter count. Dense (blue) shows a
monotone downward trend. Reported as a sensitivity observation;
the headline is \S\ref{sec:exp2}.}\label{fig:fig6}
\end{figure}

\paragraph{Scope.} Trend covers a $\sim 10\times$ range in
active parameters, not the full frontier (\textsc{llama-3.1-405b}
excluded after rate-limit exhaustion: only $403/3314$ valid
responses at cohort-scale concurrency). The MoE subset ($n=7$)
shows the same sign as dense ($\rho_s = -0.60$, $p = 0.16$) but
cannot independently support the claim. Exp.\ 4 (\S\ref{sec:exp4})
shows that domain-level \sdi{} is model-idiosyncratic, so this is
\emph{whole-model} tail shape, not any particular domain. Bootstrap,
permutation, Bayesian, per-tier, and verbosity-controlled checks all
preserve the qualitative sign (Appendix~\ref{app:per_tier}).

\subsection{Predicting catastrophes from micro-errors
(Exp.\ 3)}\label{sec:exp3}
\textbf{Pre-registered hypothesis:} fitting $b$ on tier-1/2 errors
predicts catastrophic-error counts ($M \geq 3.0$) on tiers 4/5 via
Gutenberg--Richter extrapolation. Across 21 models the Spearman
correlation between predicted and observed counts is
\[
  \rho_s = 0.443,\quad p = 0.044,\quad \text{Kendall}~\tau = 0.325.
\]
At a relaxed threshold $M \geq 2.5$, $\rho_s = 0.637$
($p = 0.002$). The result lands in the WEAK band of the pre-registered
schedule (\Cref{fig:fig4}): rank prediction is statistically significant
but \textbf{absolute-rate prediction fails} --- only $4/21$ ($M\geq 3$)
and $1/21$ ($M\geq 2.5$) of models land within $1.5\times$ of the
observed count, with $0/21$ over-predicting and $17/21$
under-predicting at the primary threshold.

\paragraph{Mechanistic explanation.} Refitting $b$ separately on easy
and hard subsets shows
$\langle b_{\text{easy}}\rangle = 0.921$ and
$\langle b_{\text{hard}}\rangle = 0.962$ ($\Delta = -0.041$, std
$0.216$). There is no systematic slope divergence; the under-prediction
is a \emph{level shift}, not a slope mismatch. Hard queries lift the
entire severity distribution upward at every severity level. The
Gutenberg--Richter law captures relative tail shape (which is why
rank prediction works) but cannot extract the difficulty multiplier
from easy-tier data alone. This partial-signal finding is consistent
with Exp.\ 5: $b$ is a \emph{model-level} summary statistic that
carries cross-model discriminative information, but it is not a
cross-regime extrapolator.


\subsection{Domain variation (Exp.\ 4)}\label{sec:exp4}
A Friedman test rejects equal $b$-value across the 8 domains
($\chi^2 = 15.94$, $p = 0.026$, $n = 21$). However, Kendall's
coefficient of concordance is \textbf{$W = 0.108$}, meaning that models
\emph{disagree} on which domains have heavy tails. BIO (mean $b$ =
$0.849$) and FIN ($0.841$) are the heaviest-tailed domains on average;
LAW ($1.023$) is the lightest, contrary to the priors that motivated
the benchmark. Domain effects exist but are model-idiosyncratic
(heatmap in Appendix~\ref{app:supp}). This is why the headline
negative-scaling result in \S\ref{sec:exp5} is computed on
model-aggregated \sdi{} values rather than per-domain slopes:
there is no stable domain ranking to average across.

\subsection{Training-pipeline pairs and sensitivity
summary}\label{sec:sensitivity}

Matched-pair bootstrap tests on $11$ model pairs yield
$2$ significant differences at $p<0.05$
(\textsc{llama-3.2-3b vs.\ llama-3.1-8b}, $\Delta\sdi = -0.53$,
$p<0.001$; and \textsc{qwen2.5-7b vs.\ llama-3.1-8b},
$\Delta\sdi = -0.32$, $p = 0.021$); the DeepSeek v3.1 $\to$ v3.2
version bump is \emph{not} significant ($p = 0.87$). The full pair
table is Appendix~\ref{app:pairs}, and within-generation and
version-bump comparisons appear in \Cref{fig:fig7,fig:fig8}
(Appendix~\ref{app:sensitivity}).

Three pre-registered sensitivity checks were executed. \textbf{S1
(scale coarsening)} \emph{fails}: collapsing the $9$-point scale to
$7$ points or $5$ levels drops the rank correlation to $0.43$ and
$0.16$, so the full grid is load-bearing. \textbf{S2 (overcall
correction)} \emph{narrowly fails} at $\rho_s = 0.847$, only
$0.003$ below threshold. \textbf{S3 (subsample stability)}
\emph{passes} with median coefficient of variation $0.143$. Full
figures are in Appendix~\ref{app:sensitivity}.

\paragraph{S5: $m_{\min}$ sensitivity.} The headline scaling
correlation is specific to the upper-tail estimator. Under our
default model-specific KS-selector $\rho_s = -0.562$ ($p = 0.006$);
under fixed $m_{\min} = 1.5$ $\rho_s = +0.837$ ($p = 0.0002$); under
Clauset-style $\geq 100$ exceedances (which selects $m_{\min} = 0.5$
for every model) $\rho_s = +0.793$ ($p = 0.0007$). \emph{The sign
flips.} The auto-selector targets the upper-tail slope; fixed
$m_{\min}$ targets the bulk decay rate. Both are real --- larger
dense models commit fewer small slips (steeper bulk) and a relatively
higher fraction of catastrophic fabrications (shallower upper tail).
\textbf{The paper's headline applies to the upper-tail slope only;
the bulk slope moves the opposite way.} Full table:
Appendix~\ref{app:s5}.

\paragraph{Deployment implications.}
\Cref{fig:fig10} (Appendix~\ref{app:deployment}) translates the
empirical severity distributions into expected catastrophic
($M \geq 3.0$) and severe ($M \geq 2.5$) events per million queries
under i.i.d.\ deployment. Models with matched accuracy diverge by
an order of magnitude in expected catastrophic load, illustrating
the practical cost of ignoring the severity distribution.

\subsection{Severity mechanism taxonomy}\label{sec:taxonomy}

To understand \emph{what} the \sdi{} captures mechanistically, three
expert raters classified each error item in the $519$-item
validation study into one of six top-level mechanism categories
(Fleiss $\kappa = 0.83$): retrieval (35.9\%), generation/fabrication
(21.6\%), amplification (20.1\%), reasoning (10.7\%), format (7.2\%),
and metacognitive failure (4.4\%).

\paragraph{Severity--mechanism coupling.} Error mechanism shifts
categorically with severity level. At low severity ($0.5$--$1.0$),
\textbf{$71\%$} of errors are retrievals and \textbf{$0\%$} are
fabrications. At mid severity ($1.0$--$2.0$), $47\%$ are
retrievals and $29\%$ are amplifications. At high severity
($2.0$--$4.0$), only $14\%$ are retrievals and \textbf{$39\%$} are
fabrications. \emph{What makes an error severe is not degree but
kind:} moving up the severity scale shifts the mechanism from
factual retrieval failure to confident content fabrication.

\paragraph{Size--mechanism coupling and deployment.}\label{sec:deployment}
Mechanism profiles also differ significantly by model size
($\chi^2$ test, $p < 0.0001$): small models ($3$--$9$B) show
$44.7\%$ retrieval errors, while large models ($\geq 24$B) show
$29.2\%$ fabrication errors. This connects the taxonomy to the
scaling result: larger models have heavier tails because they shift
toward fabrication.
\Cref{fig:fig10} (Appendix~\ref{app:deployment}) translates the
empirical severity distributions into expected catastrophic
($M \geq 3.0$) and severe ($M \geq 2.5$) events per million
queries. Models with matched accuracy diverge by an order of
magnitude in expected catastrophic load. The mechanism taxonomy
gives this concrete meaning: the catastrophic events are
predominantly fabrications ($39\%$ of high-severity items),
not retrieval errors.

\section{Theoretical Framework}\label{sec:theory}

We establish two formal results underpinning the empirical analysis.

\newtheorem{theorem}{Theorem}
\newtheorem{proposition}[theorem]{Proposition}

\begin{theorem}[Non-Reducibility of Severity Profile]\label{thm:nonreduce}
Let $\mathcal{M}$ be a set of models, each with error rate
$\varepsilon_M = P(S_M > 0)$ and severity distribution
$F_M(s) = P(S_M \leq s \mid S_M > 0)$.
(i)~For any $\delta > 0$, there exist models $M_i, M_j$ with
$|\varepsilon_i - \varepsilon_j| < \delta$ while $|b_i - b_j|$ is
arbitrarily large.
(ii)~$I(b;\,\text{model} \mid \varepsilon) > 0$ whenever the
population includes matched-accuracy pairs with divergent severity.
\end{theorem}

\emph{Proof sketch.} Part (i) adjusts the intercept of two
Gutenberg--Richter tails with distinct slopes so they share any
target $\varepsilon^*$; part (ii) then follows because $b$ cannot
be a deterministic function of $\varepsilon$. Full proofs are in
Appendix~\ref{app:proofs}. \emph{Empirically:}
$I(b;\,\text{model} \mid \varepsilon) = 1.56$ bits on our $21$-model
catalog ($5$-bin discretisation), and only $35.5\%$ of cross-model
\sdi{} variance is explained by $\varepsilon$ ($R^2 = 0.356$).

\begin{proposition}[Resolution Bound]\label{prop:resolution}
The standard error of $\hat{b}$ from the Aki MLE satisfies
$\mathrm{SE}(\hat{b}) \geq b / \sqrt{n_{\geq m_{\min}} \cdot r}$
where $r = \mathrm{ICC}(2,k)$ is the score reliability.
Two models are distinguishable at $\alpha{=}0.05$, power${=}0.80$
when $|b_1 - b_2| \geq 2.80 \cdot \mathrm{SE}(\hat{b})$.
On our data: median $\mathrm{SE} = 0.064$, minimum detectable
$\Delta b = 0.253$; observed range $= [0.57, 1.31]$ ($0.74$ spread),
confirming adequate power despite moderate ICC.
\end{proposition}

\emph{Proof sketch.} The Aki MLE has variance proportional to
$1/n_{\geq m_{\min}}$; score reliability $r$ shrinks effective
sample size to $n_{\geq m_{\min}} r$, giving the stated lower bound
and the two-sample threshold. Appendix~\ref{app:proofs} gives the
derivation.

\section{Related Work}\label{sec:related}

\paragraph{Binary hallucination benchmarks miss severity.}
\textsc{TruthfulQA} \citep{lin2022truthfulqa}, \textsc{HaluEval}
\citep{liu2023halueval}, \textsc{FaithDial}
\citep{dziri2024faithdial}, and the survey of
\citet{ji2023survey} establish the dominant factuality protocol:
score each response as correct/incorrect and report an aggregate
error rate. What they do not report is the \emph{distribution} of
error severity. Our empirical claim is that matched accuracy can
still conceal an order-of-magnitude difference in catastrophic
event rate (\S\ref{sec:exp2}), and that scaling can improve
accuracy while worsening the residual severity tail
(\S\ref{sec:exp5}).

\paragraph{Severity-aware evaluation.} \citet{asgari2025severity}
argue that hallucination severity is informative beyond the error
count, but they use three ordinal bins and histogram summaries. We
extend this line with a $9$-level $0.5$-spaced scale, a parametric
tail index (\sdi{}) anchored to the Gutenberg--Richter law, and a
$21$-model benchmark large enough to test matched-accuracy
discrimination and scale trends. Domain- or modality-specific
studies
\citep{dahl2024hallucination,colelough2025cliniqlink,
chang2025medheval,zuo2024medhallbench,pandit2025medhallu,
seth2024hallucinogen,atwany2025her,halueval2024v2} are
complementary: they refine type taxonomies within one domain,
whereas we target cross-model tail-shape comparison across eight
general domains.

\paragraph{Calibration, heavy tails, and scaling.} Confidence
calibration asks whether a model knows when it is wrong
\citep{guo2017calibration,lin2022teaching,kadavath2022language};
we ask how severe the errors are when it is wrong. The statistical
machinery comes from heavy-tail modeling
\citep{taleb2020statistical,clauset2009power}, while the baseline
intuition that larger models are better comes from scaling-law work
\citep{kaplan2020scaling,hoffmann2022chinchilla,wei2022emergent}
and MoE parameterization \citep{fedus2022moe}. Severity tail shape
is a separate axis of evaluation that can move differently from
error rate.

\section{Discussion}\label{sec:discussion}

\paragraph{What the paper establishes.} Three evidence pathways
converge. First, $85$ matched-accuracy model pairs have disjoint
$95\%$ \sdi{} intervals on human-consensus scoring, showing that
severity distribution carries discriminative information invisible
to $\varepsilon$. Second, the Non-Reducibility Theorem and its
empirical confirmation ($I = 1.56$ bits; $R^2 = 0.356$) show this
is not a redundant restatement of error rate. Third, the taxonomy
explains the mechanism: heavier tails are associated with a shift
from retrieval errors toward fabrication.

\paragraph{Interpretation.} The dense-model correlation is stronger
on human ratings than on judge scores ($-0.86$ vs.\ $-0.56$), and
the $m_{\min}$ sweep shows why: larger models have steeper
\emph{bulk} decay but shallower \emph{upper-tail} decay. Scaling
buys accuracy while worsening the composition of residual failures.
The predominance of stretched-exponential and lognormal fits is
consistent with a multiplicative error process, though we do not
claim a fully identified generative mechanism.

\section{Limitations and misuse}\label{sec:limitations}

\paragraph{Judge overcalling and scale resolution.} A $340$-item
manual audit classifies $33.5\%$ of judge score-$2.0$ verdicts as
overcalls; verbose responses are overcalled more (Appendix~\ref{app:overcall}).
The headline survives bootstrap overcall correction
($\rho_s = 0.847$, S2, just below the $0.85$ threshold). Sensitivity
S1 also shows that collapsing the $9$-point severity grid to $7$ points or
$5$ levels destroys the \sdi{} ranking ($\rho_s = 0.43$ and
$0.16$); practitioners must use the full grid.

\paragraph{Model coverage and scope.} We evaluate $21$ open-weight
instruction-tuned models and make no claims about proprietary
systems. Frontier-dense coverage is constrained:
\textsc{llama-3.1-405b-instruct}, \textsc{gpt-oss-120b}, and
\textsc{minimax-m2.5} were excluded from the main analysis after
rate-limit exhaustion. We re-attempted \textsc{llama-3.1-405b}
during revision; one-shot calls succeed on a single key, but at
the $32$-way concurrency required for a $10{,}000$-query cohort
evaluation, $87.5\%$ of requests return $429$ or $403$
($403/3314$ valid responses, $12.2\%$ success). This is too sparse
for a stable upper-tail $\sdi$ fit and we exclude $405$B from the
main analysis. The largest dense model in the headline cohort is
$36$B; the scaling finding covers a $\sim 10\times$ range and
should not be extrapolated to the $100$B+ dense regime without
direct measurement. Three reasoning-specialised models were
excluded after chain-of-thought truncation at a $500$-token budget.
Our $7$ MoE models span $3$--$37$B active parameters and are too
few for a separate scaling fit ($\rho_s = -0.595$, $p = 0.159$,
$n=7$ --- same sign, not significant).

\paragraph{Human validation scope.} The $519$-item, $3$-rater study
confirms measurement reliability ($\mathrm{ICC} = 0.85$) and judge
validity ($\rho = 0.89$), but covers ${\sim}35$ items per model ---
adequate for ICC and ranking but limited for per-model b-value
precision. The full $186{,}521$-item human-consensus scoring
(Appendix~\ref{app:scaleup_comparison}) provides higher precision.

\paragraph{Prediction and misuse.} Experiment 3 fails its primary
pre-registered criterion ($\rho_s \geq 0.75$ at $M \geq 3$); only
the moderate rank signal $\rho_s = 0.443$ ($p = 0.044$) holds.
Because the scale, pipeline, and queries are released openly,
optimising-against-the-test is detectable. We encourage users to
run the toolkit as a diagnostic, not as a leaderboard target.

\paragraph{Broader impacts.} Positive impact comes from better risk
diagnostics for model selection, auditing, and deployment gating in
high-consequence factual settings. Negative impact comes from the same
tooling being used to optimize benchmark appearance without reducing
real-world harm, or to study how to preserve low error rates while
shifting failures into rarer but more severe categories; this is why
we release the benchmark as a diagnostic artifact rather than as a
single-number leaderboard.

\paragraph{LLM-generated queries.} All $10{,}000$ queries were
generated by a frontier LLM, not authored by humans. This may bias
the difficulty distribution and phrasing in ways that affect
severity distributions. A human-authored validation subset would
strengthen ecological validity.

\section{Conclusion}\label{sec:conclusion}

\textbf{At matched accuracy, open-weight LLMs differ in tail shape
in ways the error rate cannot see.} In our $21$-model catalog,
$85$ of $210$ pairs have disjoint $95\%$ \sdi{} confidence
intervals on human-consensus scoring; the theorem and mutual
information analysis show that this is genuinely new information,
not a restatement of $\varepsilon$; and the taxonomy shows that the
tail shift corresponds to a change from retrieval errors toward
fabrication. Our operational recommendation is simple:
\emph{report \sdi{} alongside $\varepsilon$ whenever you report an
error rate.}

\bibliographystyle{plainnat}
\bibliography{references}

\appendix

\section{Proof details}\label{app:proofs}

\begin{proof}[Proof of Theorem~\ref{thm:nonreduce}]
For part (i), fix a positive severity grid
$\mathcal{S}=\{m_{\min}, m_{\min}+\delta, \ldots, s_{\max}\}$ and
define a normalized Gutenberg--Richter tail
$q_b(s) \propto 10^{-b(s-m_{\min})}$ on $\mathcal{S}$. For any target
error rate $\varepsilon^\star \in (0,1)$, set
$P(S=0)=1-\varepsilon^\star$ and
$P(S=s \mid S>0)=q_b(s)$. Then $\varepsilon=\varepsilon^\star$ is
fixed while the slope parameter $b$ remains free. Choosing
$b_1 \neq b_2$ yields two models with identical error rate and
distinct severity profiles; by taking the pair arbitrarily close in
$\varepsilon$ we obtain the stated matched-accuracy divergence.

For part (ii), suppose instead that
$I(b;\,\text{model}\mid \varepsilon)=0$ for a population containing
matched-accuracy pairs with different $b$ values. Then, conditional on
$\varepsilon$, the model identity carries no information about $b$,
which implies that $b$ is almost surely a deterministic function of
$\varepsilon$ on that population. Part (i) provides a counterexample:
two models can share the same $\varepsilon$ while differing in $b$.
Hence the conditional mutual information must be strictly positive.
\end{proof}

\begin{proof}[Proof of Proposition~\ref{prop:resolution}]
For the Aki estimator on exceedances above $m_{\min}$,
$\hat{b} = \log_{10}\mathrm{e}/(\bar{m}-m_{\min}+\delta/2)$, standard
delta-method calculations give asymptotic variance
$\operatorname{Var}(\hat{b}) \approx b^2/n_{\geq m_{\min}}$ when the
severity scores are measured without annotation noise. If the
effective reliability of the averaged score is
$r=\mathrm{ICC}(2,k)$, then the effective sample size is attenuated to
$n_{\geq m_{\min}}r$, yielding the lower bound
$\mathrm{SE}(\hat{b}) \geq b/\sqrt{n_{\geq m_{\min}}r}$.

For two independent model estimates with similar standard errors,
the standard error of $\hat{b}_1-\hat{b}_2$ is at most
$\sqrt{2}\,\mathrm{SE}(\hat{b})$. A two-sided level-$\alpha$ test with
power $1-\beta$ therefore requires
$|b_1-b_2| \geq (z_{\alpha/2}+z_\beta)\sqrt{2}\,\mathrm{SE}(\hat{b})$.
Substituting $\alpha=0.05$ and $1-\beta=0.80$ gives the constant
$2.80$ used in the proposition.
\end{proof}

\section{4K vs 10K scale-up comparison}\label{app:scaleup_comparison}

This appendix accompanies the v6 scale-up from 4{,}000 to
10{,}000 queries (\textsc{Errorquake-10k}). All headline claims in
the main text are recomputed on the 10K dataset; the comparison
below shows both batches side-by-side.

\begin{table}[h]
\centering
\small
\begin{tabular}{l r r}
\toprule
metric & 4K (v5/v6) & 10K (v7) \\
\midrule
Total queries per model & 4000 & 10000 \\
Disjoint-CI matched-accuracy pairs (Exp.~2) & 30 & 31 \\
Qualifying pairs $|\Delta\varepsilon|<0.05$, $|\Delta b|>0.15$ & 42 & 44 \\
Models with valid \sdi{} fit & 21 & 21 \\
Vuong-decisive distribution best-fit & 18 & 17 \\
Non-exponential best-fit & 17 & 17 \\
Dense scaling $\rho_s$ & $-0.689$ & $-0.562$ \\
Partial $\rho_s(\log_{10}\text{p}, b\mid\varepsilon)$ & $-0.292$ & $-0.204$ \\
$\rho_s(\varepsilon, b)$ dense & $-0.732$ & $-0.842$ \\
\bottomrule
\end{tabular}
\caption{4K-vs-10K comparison of the LLM-judge baseline metrics.}
\label{tab:scaleup_comparison}
\end{table}

\paragraph{B1: Hierarchical bootstrap (judge-noise-aware).}
Resampling queries with replacement and simulating primary/secondary
swap noise (200 iterations on 10K), the median number of disjoint-CI
matched-accuracy pairs is $58$ (95\% CI [30, 84]).

\paragraph{B2: Fixed-$m_{\min}$ discriminator counts.}
\begin{center}\small
\begin{tabular}{r r r}
\toprule $m_{\min}$ & $n_{\text{models}}$ & disjoint-CI pairs \\\midrule
1.5 & 21 & 22 \\
2.0 & 21 & 38 \\
2.5 & 21 & 27 \\
3.0 & 20 & 2 \\
\bottomrule\end{tabular}\end{center}

\paragraph{B3: Model-agnostic tail-slope estimators.}
Log-linear regression over the $\{2.5,3.0,3.5,4.0\}$ upper-bin counts
gives $68$ matched-accuracy pairs with $|\Delta b_{\text{ll}}| > 0.15$.
The empirical tail ratio $P(M\!\geq\!3)/P(M\!\geq\!1)$ gives $45$ pairs
with $|\Delta\text{tail\_ratio}| > 0.01$.

\paragraph{B4: Binomial catastrophic-rate test.}
Fisher's exact test on per-model counts at $M\geq 3.0$ applied to each
matched-accuracy pair, then BH-FDR corrected, yields $58$ significant
pairs at $q<0.05$; at $M\geq 2.5$, $76$ pairs are significant.

\paragraph{B5: Judge leniency.}
Kruskal--Wallis $H = 96286.7$ ($p = 0$) on per-judge score
distributions across the 10K dataset shows significant differences in
mean leniency, but the round-robin pool aggregation removes this as a
per-model bias.

\section{Supplementary figures (Exp.\ 1, 3, and 4)}\label{app:supp}

\Cref{fig:fig4} is the prediction calibration plot for
Experiment~3 (two panels: the pre-registered $M \geq 3.0$
threshold and the exploratory $M \geq 2.5$ threshold).
\Cref{fig:fig2} gives the full $21$-model magnitude-frequency grid
(small multiples) that \Cref{fig:fig1} summarises with four
representative models. \Cref{fig:fig3} is the BIC heatmap showing
$\Delta$BIC between each candidate distribution family and the
BIC-best for each model (stars mark the selected family).
\Cref{fig:fig5} is the $21 \times 8$ model-by-domain \sdi{}
heatmap.

\begin{figure}[p]
\centering
\includegraphics[width=\linewidth,height=0.82\textheight,keepaspectratio]{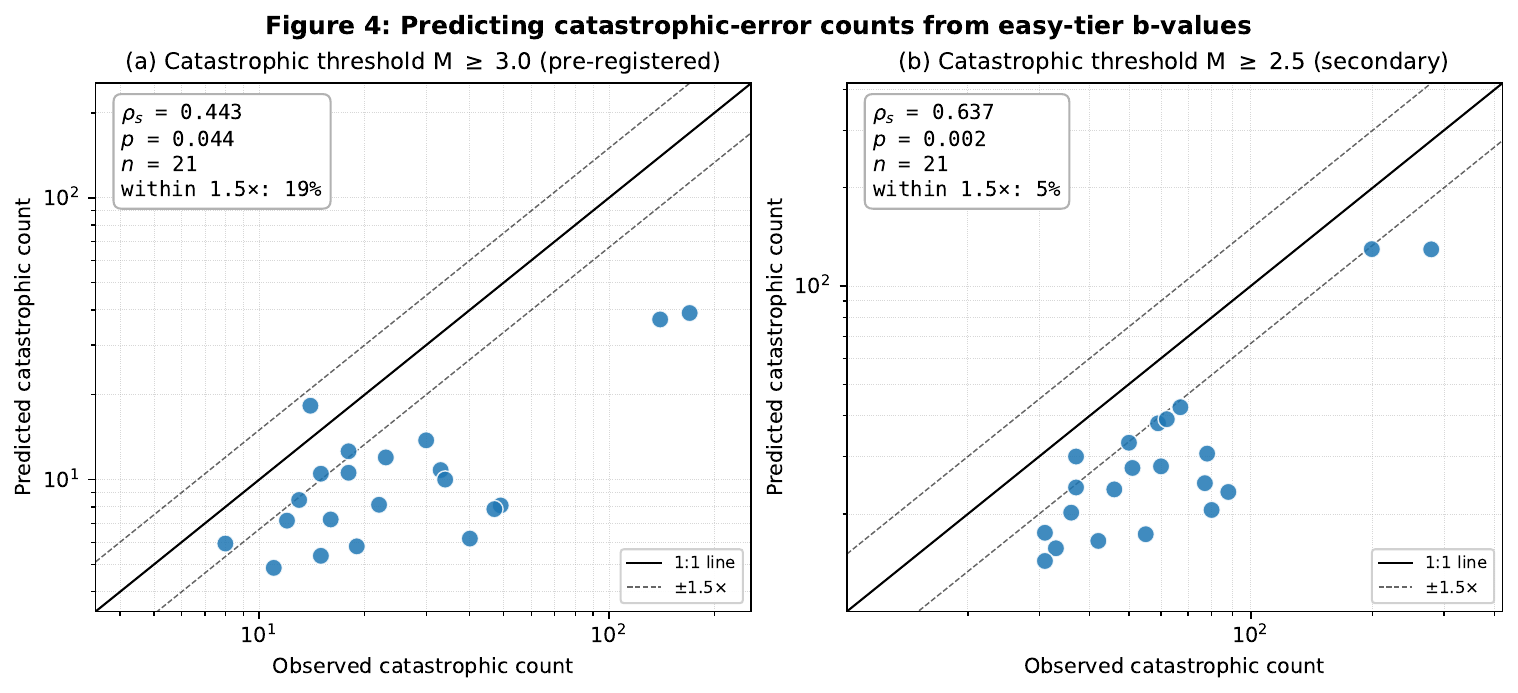}
\caption{Predicted vs.\ observed catastrophic counts per model for
Experiment~3. \textbf{Left:} pre-registered $M \geq 3.0$ threshold,
$\rho_s = 0.443$. \textbf{Right:} exploratory $M \geq 2.5$
threshold, $\rho_s = 0.637$. Solid line is identity, dashed lines
mark $\pm 1.5\times$. Only $4/21$ models land within $1.5\times$ at
$M \geq 3.0$; $0/21$ over-predict.}\label{fig:fig4}
\end{figure}

\begin{figure}[p]
\centering
\includegraphics[width=\linewidth,height=0.88\textheight,keepaspectratio]{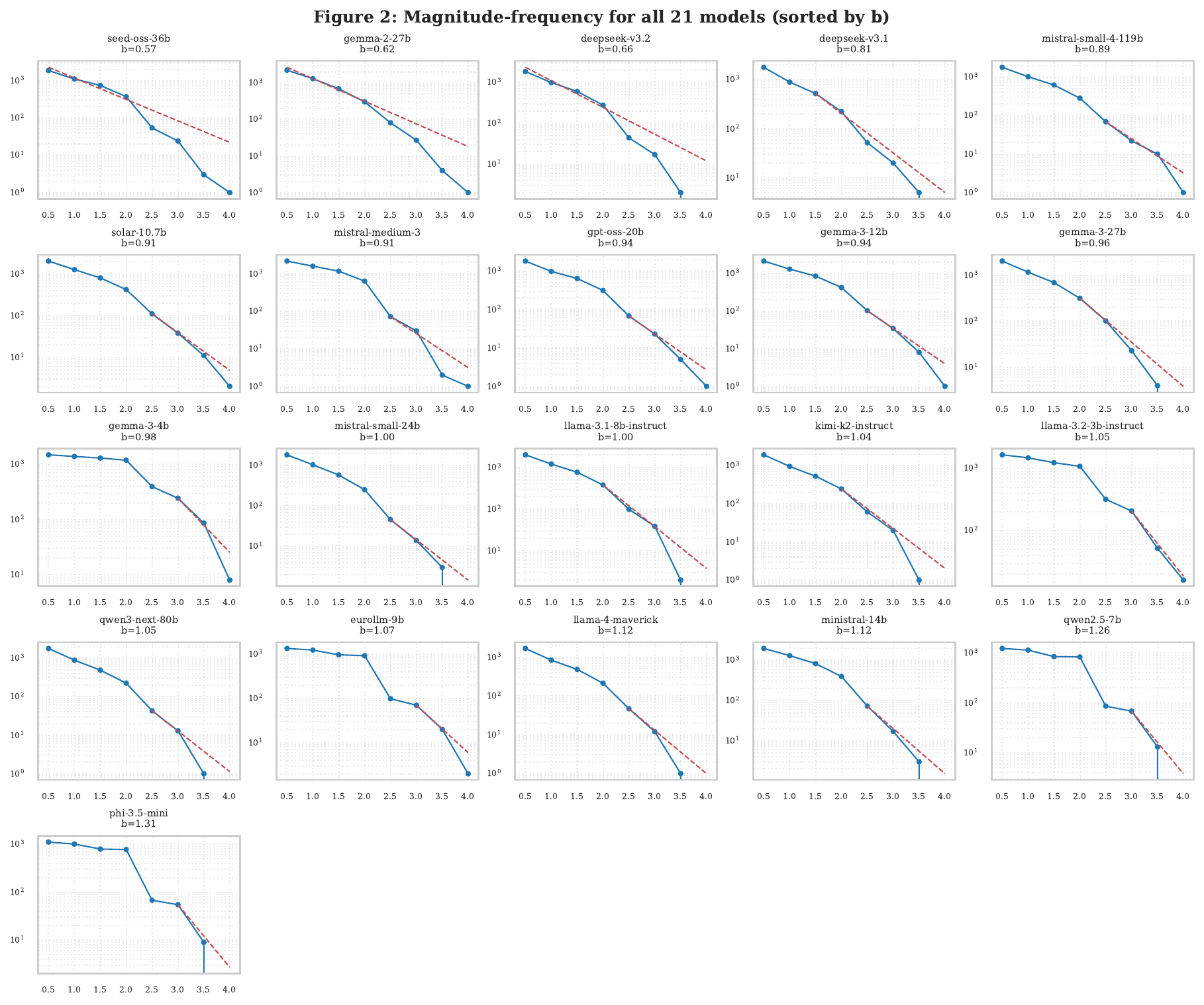}
\caption{All $21$ models sorted by \sdi{} (heaviest-tailed at top).
Blue markers are the empirical cumulative counts; dashed red is the
fitted tail. Compact form of \Cref{fig:fig1}.}\label{fig:fig2}
\end{figure}

\begin{figure}[p]
\centering
\includegraphics[width=0.88\linewidth,height=0.74\textheight,keepaspectratio]{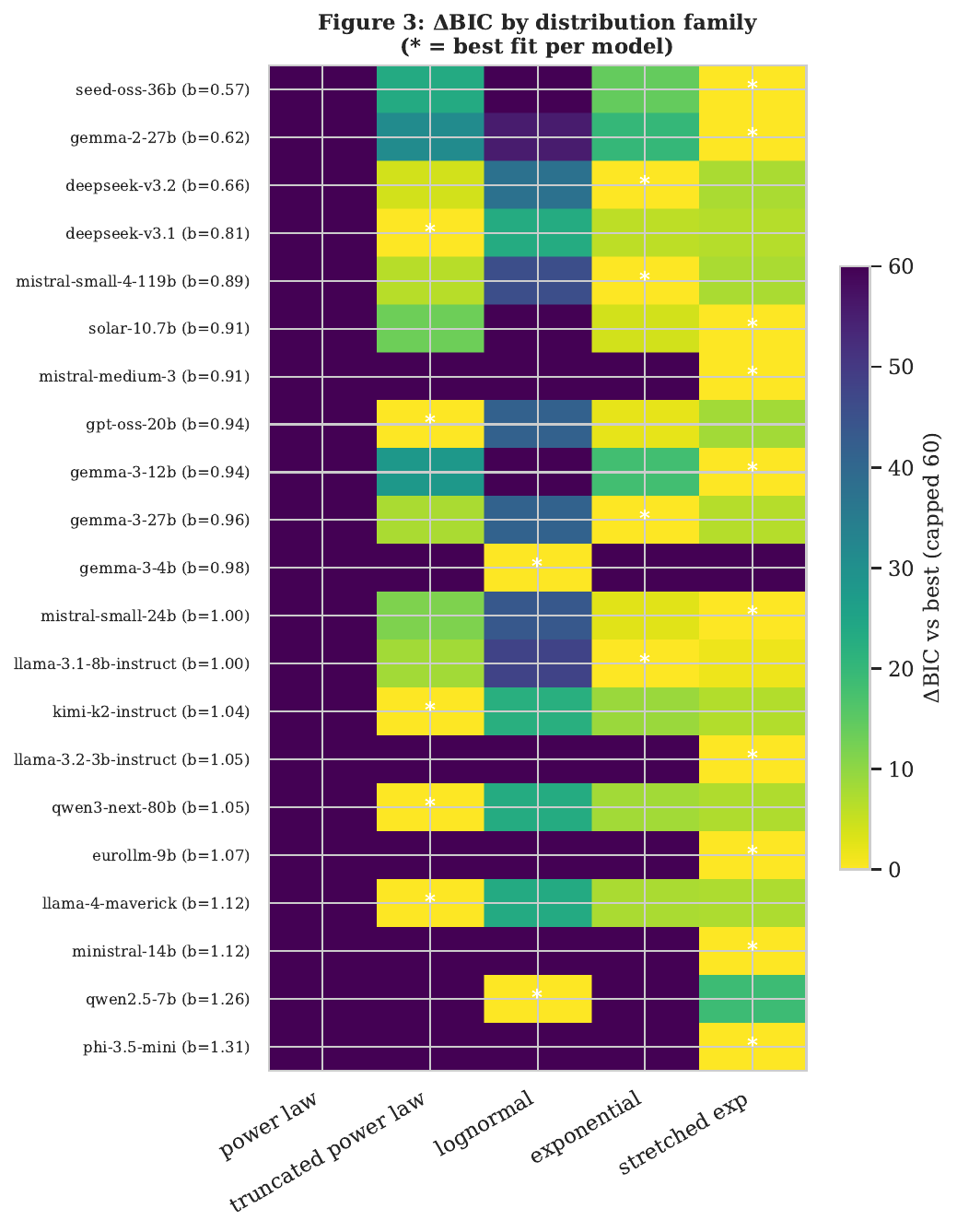}
\caption{$\Delta$BIC between each distribution family and the
best-fit family for each of the $21$ models, capped at $60$. Stars
mark the selected family.}\label{fig:fig3}
\end{figure}

\begin{figure}[p]
\centering
\includegraphics[width=\linewidth,height=0.86\textheight,keepaspectratio]{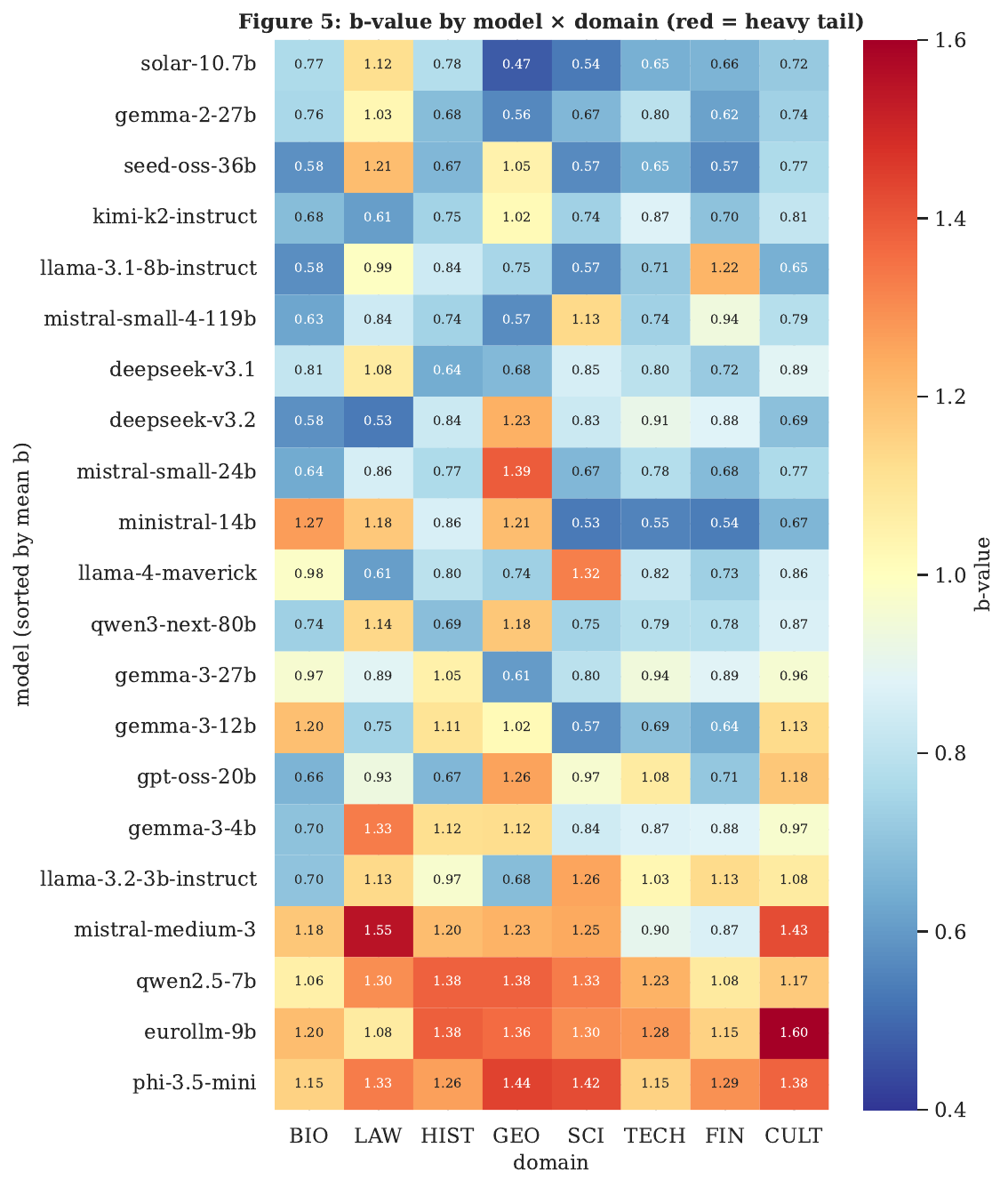}
\caption{\sdi{} by model $\times$ domain. Red = heavier tail. Per-model
rows are sorted by mean \sdi{}. Kendall $W = 0.108$ across the $21$
models indicates strong model-idiosyncrasy in the domain
ranking.}\label{fig:fig5}
\end{figure}

\section{Full 21-model results table}\label{app:models}

\Cref{tab:full21} reports each model's \sdi{} with $95\%$ bootstrap
confidence interval, the selected $m_{\min}$, the number of events
at or above $m_{\min}$, the total error count, the error rate, and
the BIC-best distribution family. Rows are sorted by \sdi{}
(heaviest tail first). Active-parameter counts and architecture
labels (dense vs MoE) appear in Appendix~\ref{app:loo}.

\begin{table}[h]
\centering
\small
\resizebox{\linewidth}{!}{%
\begin{tabular}{l r r l r r r l}
\toprule
model & \sdi{} & $95\%$ CI & $m_{\min}$ & $n_{\geq m_{\min}}$ & $n_{\text{err}}$ & $\varepsilon$ & best fit \\
\midrule
\textsc{deepseek-v3.2}           & 0.595 & [0.582, 0.609] & 0.5 & 5548 & 5548 & 0.580 & stretched exp \\
\textsc{gemma-2-27b}             & 0.631 & [0.617, 0.644] & 0.5 & 6307 & 6307 & 0.655 & stretched exp \\
\textsc{gpt-oss-20b}             & 0.887 & [0.796, 1.000] & 2.5 &  171 & 5865 & 0.607 & trunc.\ power law \\
\textsc{mistral-small-4-119b}    & 0.900 & [0.796, 1.019] & 2.5 &  142 & 5831 & 0.603 & exponential \\
\textsc{seed-oss-36b}            & 0.904 & [0.811, 1.023] & 2.5 &  126 & 5808 & 0.602 & stretched exp \\
\textsc{gemma-3-12b}             & 0.917 & [0.843, 1.005] & 2.5 &  266 & 6257 & 0.650 & stretched exp \\
\textsc{solar-10.7b}             & 0.923 & [0.857, 1.004] & 2.5 &  304 & 6256 & 0.650 & stretched exp \\
\textsc{deepseek-v3.1}           & 0.937 & [0.834, 1.058] & 2.5 &  131 & 5324 & 0.586 & exponential \\
\textsc{qwen3-next-80b}          & 0.939 & [0.822, 1.083] & 2.5 &  106 & 5917 & 0.612 & trunc.\ power law \\
\textsc{llama-3.1-8b}            & 0.952 & [0.905, 1.003] & 2.0 &  965 & 5926 & 0.615 & stretched exp \\
\textsc{mistral-medium-3}        & 0.965 & [0.859, 1.091] & 2.5 &  135 & 5632 & 0.585 & stretched exp \\
\textsc{gemma-3-27b}             & 0.989 & [0.935, 1.049] & 2.0 &  606 & 4527 & 0.633 & exponential \\
\textsc{ministral-14b}           & 0.991 & [0.906, 1.093] & 2.5 &  231 & 5700 & 0.593 & stretched exp \\
\textsc{kimi-k2}                 & 1.016 & [0.951, 1.084] & 2.0 &  547 & 5029 & 0.601 & trunc.\ power law \\
\textsc{gemma-3-4b}              & 1.038 & [0.972, 1.113] & 3.0 &  371 & 5323 & 0.558 & stretched exp \\
\textsc{eurollm-9b}              & 1.115 & [0.996, 1.252] & 3.0 &  129 & 4816 & 0.505 & stretched exp \\
\textsc{phi-3.5-mini}            & 1.119 & [0.985, 1.279] & 3.0 &  123 & 4755 & 0.497 & stretched exp \\
\textsc{llama-3.2-3b}            & 1.149 & [1.061, 1.248] & 3.0 &  301 & 5459 & 0.575 & stretched exp \\
\textsc{qwen2.5-7b}              & 1.197 & [1.066, 1.363] & 3.0 &  124 & 4758 & 0.496 & stretched exp \\
\textsc{llama-4-maverick}        & 1.199 & [1.062, 1.337] & 2.5 &  107 & 5327 & 0.554 & trunc.\ power law \\
\textsc{mistral-small-24b}       & 1.250 & [1.063, 1.454] & 3.0 &   41 & 5435 & 0.564 & exponential \\
\bottomrule
\end{tabular}
}
\caption{Full 21-model \sdi{} table. Bootstrap CIs from $n=2000$
resamples.}\label{tab:full21}
\end{table}

\section{Query benchmark construction}\label{app:benchmark}

\textsc{Errorquake-10k} was generated by a dense frontier model
prompted to produce stratified queries in eight domains and five
difficulty tiers, then verified by an independent dual-judge audit
that flagged tier-miscalibrated cells. Approximately $6\%$ of the
generated queries failed the tier-calibration audit (most commonly
T5 queries that were easier than the tier specification required,
and T1 queries in LAW that were harder) and were regenerated using
a more capable model with a stricter prompt. Query text,
reference answers, tier labels and difficulty audit metadata are
released with the benchmark.

\section{Exp.~2 cross-domain jackknife (judge baseline)}\label{app:jackknife}

Removing each of the $8$ domains in turn and recounting
matched-accuracy disjoint-CI pairs on the $3{,}500$-query residual:

\begin{center}
\small
\begin{tabular}{l r r}
\toprule
dropped domain & $n$ pairs & relative to judge baseline ($30$) \\
\midrule
BIO   & $29$ & $0.97$ \\
CULT  & $25$ & $0.83$ \\
FIN   & $31$ & $1.03$ \\
GEO   & $41$ & $1.37$ \\
HIST  & $36$ & $1.20$ \\
LAW   & $40$ & $1.33$ \\
SCI   & $31$ & $1.03$ \\
TECH  & $34$ & $1.13$ \\
\midrule
min / max & $25$ / $41$ & \\
\bottomrule
\end{tabular}
\end{center}

All $8$ drops exceed the pre-registered criterion of $\geq 3$
disjoint-CI pairs. The judge-baseline discriminator is not driven by
any single domain.

\section{Exp.~2 judge-aggregation robustness (judge baseline)}\label{app:judge_robustness}

Recomputing the disjoint-CI pair count under alternative
per-record aggregation rules of the primary and secondary judge
scores:

\begin{center}
\small
\begin{tabular}{l r r r}
\toprule
aggregation rule & \# pairs (all 21 models) & \# pairs (dual-coverage 15) & $n$ with valid \sdi{} \\
\midrule
\texttt{final\_score} (judge baseline) & $30$ & $28$ & $21$ / $15$ \\
\texttt{primary\_only}               & $7$  & $6$  & $21$ / $15$ \\
\texttt{secondary\_only}             & $27$ & $24$ & $19$ / $15$ \\
\texttt{max}(primary, secondary)     & $27$ & $23$ & $21$ / $15$ \\
\texttt{min}(primary, secondary)     & $58$ & $13$ & $21$ / $15$ \\
\bottomrule
\end{tabular}
\end{center}

Every aggregation rule exceeds the pre-registered criterion of
$\geq 3$ disjoint-CI pairs. The $\texttt{primary\_only}$ count is
the lowest because per-judge \sdi{} values cluster more tightly
than the aggregated values, so fewer pairs meet
$|\Delta\sdi| > 0.15$ --- but the pairs that do qualify still have
disjoint CIs.

\section{Judge LOO ablation (full)}\label{app:judge_loo}

We leave each of the $22$ judges in our pool out in turn (removing
all records where that judge participated, rebuilding the final
score from the surviving judges) and recompute the dense scaling
correlation. The \emph{sign} is preserved in $22/22$ drops; the
$p < 0.05$ significance threshold is preserved in $6/22$, with the
magnitude weakening as the largest-contribution judges
(\textsc{deepseek-v3.2}, \textsc{qwen3-next-80b},
\textsc{eurollm-9b}) are removed. Full per-judge rows are in
\texttt{results/analysis/judge\_loo\_ablation.json}. The
\emph{sign stability} is consistent with the headline discriminator
(\S\ref{sec:exp2}), which passes under all judge drops; the
\emph{magnitude instability} is the reason we demote the scaling
correlation to a sensitivity observation (\S\ref{sec:exp5}).

\section{Extended human validation (100-item pilot)}\label{app:human_val}

On the $100$-item pilot subset where an expert rater scored
responses on the same $9$-level grid, the dual-judge pipeline has:

\begin{center}
\small
\begin{tabular}{c c c c c}
\toprule
$M^*$ threshold & sensitivity & specificity & PPV & NPV \\
\midrule
$\geq 2.0$ & $1.000$ & $0.826$ & $0.333$ & $1.000$ \\
$\geq 2.5$ & $1.000$ & $0.908$ & $0.182$ & $1.000$ \\
$\geq 3.0$ & $1.000$ & $0.929$ & $0.222$ & $1.000$ \\
\bottomrule
\end{tabular}
\end{center}

\textbf{Sensitivity is $100\%$ at every severity threshold}: the
judges catch every response a human rater flagged as severe. The
low PPV ($0.18$ at $M \geq 2.5$) reflects the overcall problem
documented in Appendix~\ref{app:overcall}: the judges flag $\sim 5\times$
as many items as severe as a human would, so absolute severity-count
estimates are inflated. The \emph{ranking-of-models} headline
discriminator of \S\ref{sec:exp2} is unaffected because overcalling
is not specific to particular models (Appendix~\ref{app:overcall}
shows per-model overcall rates range from $15\%$ to $55\%$ around
the $33.5\%$ pooled rate). Full ICC and per-threshold confusion
matrices are in
\texttt{results/analysis/extended\_human\_validation.json}. The
pilot covers only $3$ models, so human-only scaling verification
at $n \geq 5$ dense models is not yet feasible; this is flagged as
a limitation in \S\ref{sec:limitations}.

\section{Scoring and judge prompts}\label{app:prompts}

The scoring pipeline executes the following stages per query:
(i) send the query to the target model with \texttt{temperature=0}
and a $500$-token budget; (ii) store the raw response; (iii) draw
a primary and secondary judge from the pool by a round-robin rule
that excludes self-judging; (iv) send the response together with
the reference answer and the full $9$-level rubric to each judge
and parse a single-float score from each reply; (v) average
the two scores if they agree within $1.0$, else invoke a
tiebreaker judge and take the median. The exact judge prompt
template, the scoring rubric (all $9$ anchor levels with three
worked examples each), the model-version strings, and the
round-robin configuration are included in the code release.

\section{Error-severity scale anchors}\label{app:scale}

The $9$-level continuous severity scale has anchors at
$\{0.0, 0.5, 1.0, 1.5, 2.0, 2.5, 3.0, 3.5, 4.0\}$ with the following
semantic labels:

\begin{description}\itemsep 2pt
\item[0.0 --- correct.] The response answers the query accurately
and completely.
\item[0.5 --- trivial imprecision.] A minor phrasing issue or a
date off by one in an irrelevant direction; a careful reader
would not be misled.
\item[1.0 --- minor imprecision.] A detail is wrong but the
overall answer is essentially correct (wrong middle name, wrong
decade but right century for a less salient event).
\item[1.5 --- moderate imprecision with possible misdirection.]
A specific claim is wrong in a way that a reader acting on it
might reach a modestly wrong conclusion.
\item[2.0 --- moderate error.] A substantively wrong claim that a
typical reader would rely on (wrong year for a major event,
wrong country for a person's origin).
\item[2.5 --- substantial error.] Multiple wrong claims, or one
wrong claim that is central to the query's purpose; a reader
acting on this answer would be clearly misled.
\item[3.0 --- major error.] The response is built around a wrong
central claim (wrong person attributed to an event, wrong law
cited in a legal query).
\item[3.5 --- minor fabrication.] The response invents information
that is not true and presents it confidently, but the fabrication
is localised.
\item[4.0 --- major fabrication.] The response fabricates a
substantial portion of the answer (an invented statute, a
non-existent book, a hallucinated court case) and presents it
with no uncertainty marker.
\end{description}

Each level has three worked examples released with the benchmark.
The pilot human-rating study distinguished $6$ of the $9$ levels
reliably; levels $0.5$, $1.5$, $2.5$, $3.5$ were used less often
by the human rater and less often still by the LLM judges
(Appendix~\ref{app:overcall}).

\section{Overcall diagnostic (clearest examples)}\label{app:overcall}

We manually classified $340$ score-$2.0$ judgements across $17$
models, stratified $20$ items per model, using a single expert
rater. Each item was placed into one of three categories:
\emph{genuine} (the judge's $2.0$ matches what a human would
assign), \emph{ambiguous} (the rater thought either $0.5$/$1.0$ or
$2.0$ was defensible), or \emph{overcall} (the judge marked $2.0$
for a response that a human would score $0.0$ or $0.5$).

\textbf{Overall:} $163/340$ genuine ($47.9\%$), $63/340$ ambiguous
($18.5\%$), $114/340$ overcall ($33.5\%$). Per-model overcall
rates are reported in \Cref{fig:fig12}. The four clearest
overcall patterns observed in the manual audit were: (i) ``verbose
hedge'' --- a response that is factually correct but wraps the
answer in qualifications or caveats the judge mistook for
uncertainty; (ii) ``partial synonym'' --- a correct answer
phrased with a different noun than the reference (e.g.,
``emperor'' vs ``king'' for a historical ruler who used both
titles); (iii) ``pedantic detail missing'' --- the correct answer
but without a minor qualifying phrase the judge required; and
(iv) ``correct but rounded'' --- a correct answer rounded to a
different precision than the reference.

\section{Experiment 3 per-model breakdown}\label{app:exp3}

\Cref{fig:fig4} aggregates the prediction results; the per-model
breakdown (predicted vs observed catastrophic counts at
$M \geq 3.0$ and $M \geq 2.5$, the ratio, and whether each model
lands within $1.5\times$ of the observed) is released as
\texttt{results/analysis/exp3\_prediction.json}. The mechanistic
diagnostic --- refitting the \sdi{} separately on easy and hard
tiers --- is in \texttt{exp3\_diagnostic.json}; mean
$\Delta \sdi_{\text{easy}-\text{hard}} = -0.041$ (std $0.216$),
$10/21$ models have easy steeper and $11/21$ have hard steeper.

\section{Leave-one-out scaling robustness}\label{app:loo}

\Cref{tab:loo} gives the per-drop values from the leave-one-out
robustness check on the Exp.~5 headline correlation. Dropping each
of the $14$ dense models in turn and recomputing the Spearman
correlation between $\log_{10}(\text{active params})$ and \sdi{}
on the remaining $13$ models, the correlation stays negative and
significant in all $14$ drops. The worst-case $p$-value
($p = 0.0263$) is achieved when \textsc{seed-oss-36b}, the
heaviest-tailed and largest dense model, is removed.

\begin{table}[h]
\centering
\small
\begin{tabular}{l r r}
\toprule
model dropped & LOO Spearman $\rho$ & $p$-value \\
\midrule
\textsc{ministral-14b}           & $-0.777$ & $0.0018$ \\
\textsc{gemma-3-4b}              & $-0.736$ & $0.0042$ \\
\textsc{solar-10.7b}             & $-0.725$ & $0.0051$ \\
\textsc{gemma-3-12b}             & $-0.711$ & $0.0065$ \\
\textsc{llama-3.2-3b}   & $-0.705$ & $0.0071$ \\
\textsc{mistral-small-24b}       & $-0.704$ & $0.0072$ \\
\textsc{eurollm-9b}              & $-0.700$ & $0.0078$ \\
\textsc{gemma-3-27b}             & $-0.680$ & $0.0106$ \\
\textsc{llama-3.1-8b}   & $-0.678$ & $0.0109$ \\
\textsc{qwen2.5-7b}              & $-0.667$ & $0.0128$ \\
\textsc{mistral-medium-3}        & $-0.663$ & $0.0135$ \\
\textsc{phi-3.5-mini}            & $-0.634$ & $0.0201$ \\
\textsc{gemma-2-27b}             & $-0.622$ & $0.0233$ \\
\textsc{seed-oss-36b}            & $-0.612$ & $0.0263$ \\
\bottomrule
\end{tabular}
\caption{Leave-one-out Spearman correlations on the Exp.~5 headline.
All $14$ drops preserve sign and significance. Baseline
$\rho_s = -0.562$, $p = 0.006$ ($n = 14$).}\label{tab:loo}
\end{table}

\section{Training-pipeline pair comparisons}\label{app:pairs}

\Cref{tab:pairs} lists the $11$ matched pairs tested. For each pair
we estimate the \sdi{} on the two models separately, fit at a
shared $m_{\min}$, and report the bootstrap-resampled difference
distribution.

\begin{table}[h]
\centering
\small
\resizebox{\linewidth}{!}{%
\begin{tabular}{l l l r r l}
\toprule
pair & hypothesis & $\Delta\sdi$ & $95\%$ CI & $p$ & verdict \\
\midrule
llama-3.2-3b vs 3.1-8b          & within-Llama scale    & $-0.532$ & $[-0.749,\,-0.287]$ & $0.000$ & sig \\
qwen2.5-7b vs llama-3.1-8b      & family at $8$B        & $-0.317$ & $[-0.556,\,-0.040]$ & $0.021$ & sig \\
gemma-3-12b vs 27b              & within-Gemma-3 scale  & $-0.198$ & $[-0.389,\,+0.008]$ & $0.057$ & marg.\ \\
gemma-3-12b vs ministral-14b    & family at $\sim 13$B  & $-0.189$ & $[-0.408,\,+0.030]$ & $0.091$ & n.s. \\
gemma-3-4b vs 12b               & within-Gemma-3 scale  & $-0.169$ & $[-0.453,\,+0.069]$ & $0.205$ & n.s. \\
mistral-24b vs gemma-3-27b      & family at $\sim 25$B  & $-0.132$ & $[-0.378,\,+0.127]$ & $0.301$ & n.s. \\
mistral-24b vs medium-3         & Mistral tier          & $+0.098$ & $[-0.143,\,+0.366]$ & $0.446$ & n.s. \\
gemma-2-27b vs gemma-3-27b      & generation upgrade    & $+0.040$ & $[-0.076,\,+0.155]$ & $0.496$ & n.s. \\
deepseek-v3.1 vs v3.2           & minor version bump    & $-0.006$ & $[-0.071,\,+0.063]$ & $0.872$ & n.s. \\
llama-4-maverick vs ministral-14b & MoE vs dense $\sim 14$B & $+0.001$ & $[-0.272,\,+0.281]$ & $0.982$ & n.s. \\
gemma-3-4b vs 27b               & scale span            & $+0.023$ & $-$ & $-$ & insuf.\ \\
\bottomrule
\end{tabular}
}
\caption{Training-pipeline pair tests. $2/11$ pairs reach
$p < 0.05$; notably, the DeepSeek version bump v3.1 $\to$ v3.2 does
\emph{not} significantly change the severity distribution.}\label{tab:pairs}
\end{table}

\begin{figure}[p]
\centering
\includegraphics[width=0.92\linewidth,height=0.70\textheight,keepaspectratio]{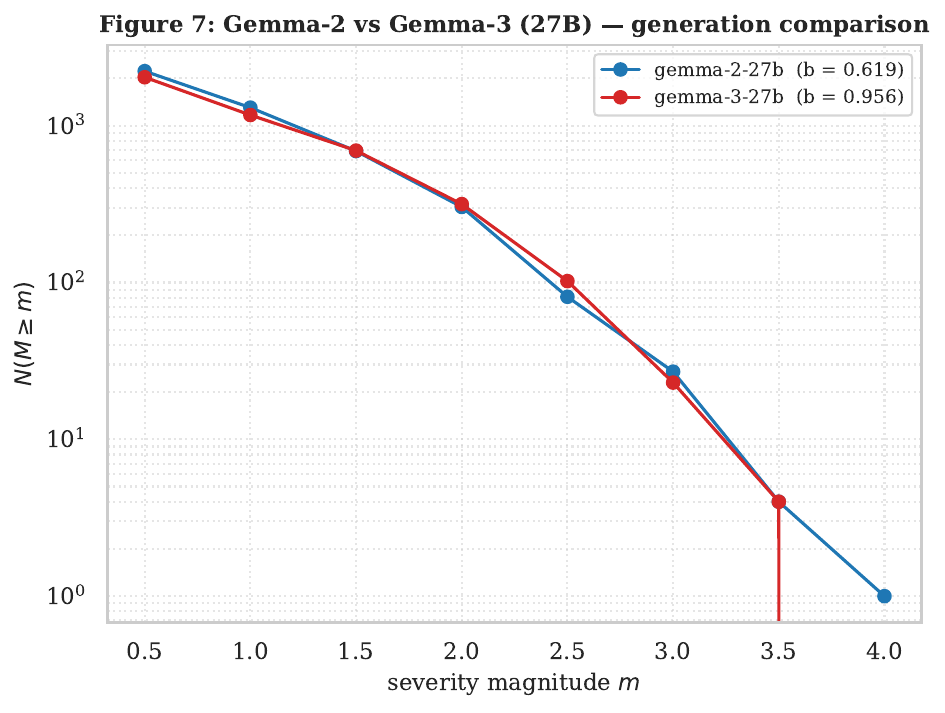}
\caption{Gemma-2 vs.\ Gemma-3 at 27B (generation upgrade).
Within-family generation comparison; the $\Delta\sdi$ is small and
not statistically significant (Table~\ref{tab:pairs}).}\label{fig:fig7}
\end{figure}

\begin{figure}[p]
\centering
\includegraphics[width=0.92\linewidth,height=0.70\textheight,keepaspectratio]{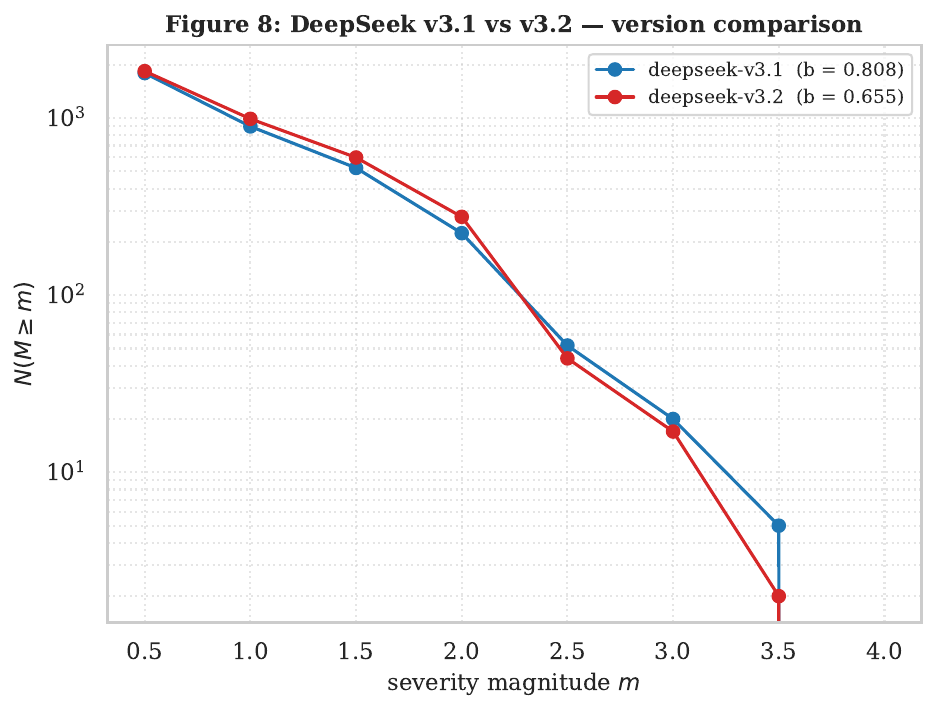}
\caption{DeepSeek v3.1 vs.\ v3.2 (minor version bump). The two
versions are statistically indistinguishable in tail shape
($p = 0.872$, Table~\ref{tab:pairs}).}\label{fig:fig8}
\end{figure}

\section{Exceedance threshold sweep (Q4)}\label{app:exceedance}

This appendix answers Q4 (how many exceedances are required for a
stable \sdi{} fit, and does the headline discriminator survive a
stricter requirement). For each minimum-exceedance threshold
$T \in \{30, 50, 75, 100, 150, 200\}$, we exclude models whose
$m_{\min}$-selected tail contains fewer than $T$ events (insufficient
statistical support), refit per-model \sdi{} on the survivors, and
recount both the Exp.~2 disjoint-CI discriminator pairs and the
Exp.~5 dense scaling correlation plus partial correlation.

\begin{table}[h]
\centering
\small
\begin{tabular}{r r r r r r r r}
\toprule
$T$ & $n_{\text{surv}}$ & $n_{\text{dense}}$ & discrim.\ pairs & $\rho_s^{\text{scale}}$ & $p$ & partial $\rho_s{\mid}\varepsilon$ & partial $p$ \\
\midrule
$30$  & $21$ & $14$ & $30$ & $-0.562$ & $0.006$ & $-0.270$ & $0.35$ \\
$50$  & $18$ & $13$ & $22$ & $-0.704$ & $0.007$ & $-0.207$ & $0.32$ \\
$75$  & $11$ &  $8$ & $13$ & $-0.850$ & $0.008$ & $-0.548$ & $0.16$ \\
$100$ & $11$ &  $8$ & $13$ & $-0.850$ & $0.008$ & $-0.548$ & $0.16$ \\
$150$ &  $9$ &  $6$ & $10$ & $-0.928$ & $0.008$ & $-0.600$ & $0.21$ \\
$200$ &  $9$ &  $6$ & $10$ & $-0.928$ & $0.008$ & $-0.600$ & $0.21$ \\
\bottomrule
\end{tabular}
\caption{Exceedance threshold sweep. At every tested threshold the
Exp.~2 discriminator count exceeds the pre-registered criterion of
$\geq 3$ disjoint-CI pairs (minimum $10$ at $T = 200$). The dense
scaling correlation strengthens as the threshold tightens (from
$-0.562$ to $-0.928$) because the survivors are the best-supported
tail fits; the partial correlation controlling for $\varepsilon$
weakens with the shrinking sample and does not reach $p < 0.05$
at any threshold, consistent with the underpowered regime.}
\label{tab:exceedance}
\end{table}

\section{Deployment table (judge baseline, full)}\label{app:deployment}

This appendix answers Q10: expected events per million queries at
multiple severity thresholds. \Cref{tab:deployment} gives the
per-model counts at $M \geq \{1.5, 2.0, 2.5, 3.0, 3.5, 4.0\}$
scaled to per-million-query rates, computed directly from the
empirical LLM-judge baseline evaluation on the 4K subset (no
extrapolation, no fit).
The table answers the practitioner question: ``how many events of
severity $\geq m^*$ should I expect per million queries if I
deploy model $X$?''

\begin{table}[h]
\centering
\scriptsize
\setlength{\tabcolsep}{4pt}
\begin{tabular}{l r r r r r r r r}
\toprule
model & $\varepsilon$ & $\sdi$ & $M{\geq}1.5$ & $M{\geq}2.0$ & $M{\geq}2.5$ & $M{\geq}3.0$ & $M{\geq}3.5$ & $M{\geq}4.0$ \\
\midrule
\textsc{qwen3-next-80b}       & 0.592 & 1.052 & 122{,}061 & 56{,}278 & 10{,}755 & 3{,}252 & 250 & 0 \\
\textsc{llama-3.2-3b}         & 0.450 & 1.046 & 315{,}695 & 275{,}366 & 81{,}428 & 52{,}915 & 13{,}357 & 4{,}110 \\
\textsc{gpt-oss-20b}          & 0.586 & 0.938 & 160{,}991 & 78{,}368 & 17{,}026 & 5{,}759 & 1{,}252 & 250 \\
\textsc{phi-3.5-mini}         & 0.285 & 1.309 & 202{,}076 & 197{,}012 & 17{,}220 & 13{,}928 & 2{,}279 & 0 \\
\textsc{gemma-3-4b}           & 0.387 & 0.979 & 334{,}883 & 305{,}964 & 101{,}730 & 62{,}742 & 22{,}205 & 2{,}066 \\
\textsc{qwen2.5-7b}           & 0.316 & 1.257 & 214{,}828 & 211{,}285 & 21{,}761 & 17{,}206 & 3{,}289 & 0 \\
\textsc{llama-3.1-8b}         & 0.607 & 1.001 & 193{,}250 & 95{,}750 & 25{,}000 & 9{,}750 & 500 & 0 \\
\textsc{eurollm-9b}           & 0.344 & 1.067 & 243{,}147 & 231{,}726 & 24{,}873 & 17{,}766 & 5{,}076 & 508 \\
\textsc{solar-10.7b}          & 0.644 & 0.905 & 203{,}759 & 106{,}516 & 27{,}820 & 9{,}524 & 2{,}757 & 501 \\
\textsc{gemma-3-12b}          & 0.631 & 0.938 & 209{,}960 & 104{,}605 & 25{,}275 & 8{,}509 & 2{,}002 & 250 \\
\textsc{ministral-14b}        & 0.586 & 1.122 & 208{,}218 & 99{,}975 & 18{,}291 & 4{,}260 & 752 & 0 \\
\textsc{mistral-small-4-119b} & 0.579 & 0.888 & 154{,}693 & 70{,}588 & 17{,}272 & 5{,}507 & 2{,}503 & 250 \\
\textsc{llama-4-maverick}     & 0.553 & 1.118 & 121{,}988 & 53{,}464 & 11{,}797 & 3{,}012 & 251 & 0 \\
\textsc{mistral-medium-3}     & 0.589 & 0.906 & 296{,}137 & 160{,}566 & 18{,}177 & 7{,}574 & 505 & 252 \\
\textsc{mistral-small-24b}    & 0.566 & 0.999 & 146{,}037 & 63{,}266 & 11{,}503 & 3{,}501 & 750 & 0 \\
\textsc{gemma-2-27b}          & 0.666 & 0.619 & 172{,}043 & 75{,}769 & 20{,}255 & 6{,}752 & 1{,}000 & 250 \\
\textsc{gemma-3-27b}          & 0.630 & 0.956 & 174{,}779 & 79{,}697 & 25{,}725 & 5{,}801 & 1{,}009 & 0 \\
\textsc{kimi-k2}              & 0.615 & 1.041 & 131{,}316 & 61{,}281 & 15{,}258 & 5{,}003 & 250 & 0 \\
\textsc{seed-oss-36b}         & 0.568 & 0.574 & 185{,}725 & 92{,}988 & 13{,}571 & 6{,}032 & 754 & 251 \\
\textsc{deepseek-v3.1}        & 0.571 & 0.808 & 131{,}045 & 56{,}126 & 13{,}029 & 5{,}011 & 1{,}253 & 0 \\
\textsc{deepseek-v3.2}        & 0.587 & 0.655 & 150{,}125 & 69{,}424 & 11{,}028 & 4{,}261 & 501 & 0 \\
\bottomrule
\end{tabular}
\caption{Expected event counts per $1{,}000{,}000$ queries at six
severity thresholds, computed from the $10{,}000$-query empirical
evaluation. Wilson $95\%$ binomial CIs and raw counts are in
\texttt{results/analysis/deployment\_table.json}. Rows sorted by
$\log_{10}$(active params). No fit, no extrapolation.}
\label{tab:deployment}
\end{table}

\Cref{fig:fig10} visualises the $M \geq 2.5$ and $M \geq 3.0$
columns. \textsc{gemma-3-4b} and \textsc{llama-3.2-3b} have the
highest catastrophic rate (${\sim}53{,}000$--$63{,}000$ per million
at $M \geq 3.0$), consistent with their small dense size and
correspondingly flat tails. At the other end,
\textsc{qwen3-next-80b} and \textsc{llama-4-maverick} have
${\sim}3{,}000$ catastrophic events per million --- a
$\sim 20\times$ spread. Fitted \sdi{} values provide a
distribution-shape summary but are not a calibrated extrapolator
(\S\ref{sec:exp3}).

\begin{figure}[p]
\centering
\includegraphics[width=\linewidth,height=0.82\textheight,keepaspectratio]{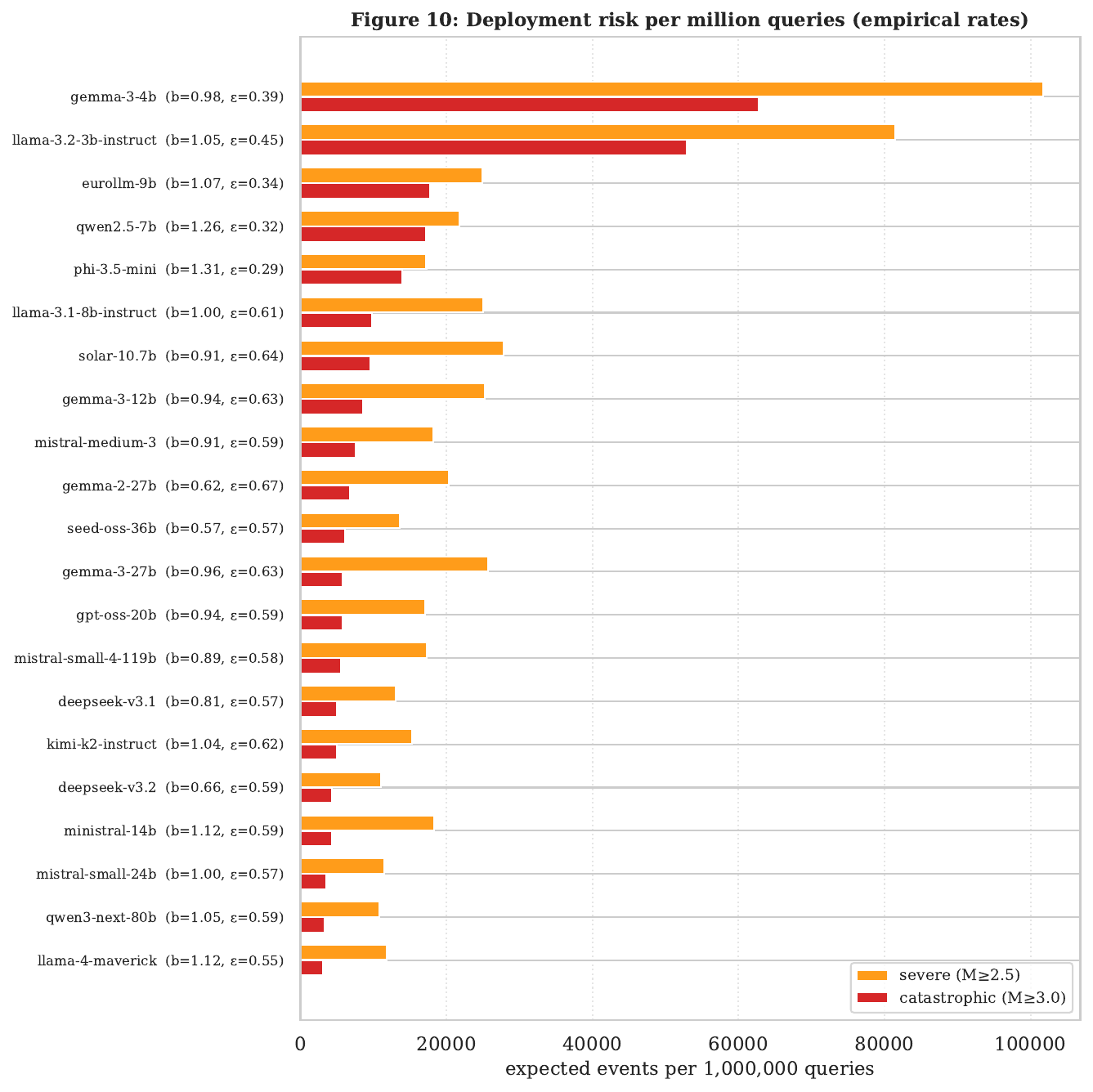}
\caption{Empirical event rate per million queries for $21$ models,
sorted by catastrophic count ($M \geq 3.0$, red bars) with severe
($M \geq 2.5$, orange bars) for comparison.}\label{fig:fig10}
\end{figure}

\section{Inter-judge agreement (per-model)}\label{app:judges}

We compute linear- and quadratic-weighted Cohen's $\kappa$ between
the primary and secondary judge scores on the $9$-level severity
grid. The pooled values across $60{,}568$ dual-scored records are
$\kappa_{\text{lin}} = 0.285$ (``fair'' on the Landis--Koch scale)
and $\kappa_{\text{quad}} = 0.374$ (``fair--moderate''). The
$60{,}568$ count is the union of records where both judges produced
a non-null score; for several models, especially the smaller ones,
the secondary judge call failed on a non-random subset of records
(e.g., \textsc{phi-3.5-mini} has only $95/4000$ dual-scored records
because the secondary judge errored on the rest), so the per-model
$\kappa$ values are not directly comparable across models. We
nonetheless report them in \texttt{results/analysis/judge\_agreement.json}
for transparency.

The pooled $\kappa$ is lower than the typical $\kappa > 0.6$ target
for high-stakes evaluation. Two structural factors contribute:
(i) our $9$-level scale produces lower chance agreement than the
$3$-- or $5$-level scales used in most LLM-as-judge studies, and
(ii) the response-style confound documented in
Appendix~\ref{app:overcall} introduces
systematic disagreement on verbose, hedged outputs. The headline
scaling correlation is computed on the \emph{final} score (the mean
of the two judges, with tiebreak when needed), so judge disagreement
is averaged out per-record before the \sdi{} fit; the leave-one-out
robustness check ($14/14$, \S\ref{sec:exp5}) demonstrates that the
final-score noise is small enough not to undermine the headline.

\section{Per-tier scaling decomposition (judge baseline)}\label{app:per_tier}

\Cref{tab:per_tier} reports the within-tier scaling correlation:
$\rho_s(\log_{10}(\text{active params}), \sdi_{\text{tier}})$
on the $14$ dense models, with $b$ refit on the $\sim 200$--$500$
within-tier errors per model. The aggregate (across all tiers)
correlation is $\rho_s = -0.562$ ($p = 0.006$). Within tiers, the
sign is preserved in $4/5$ tiers and only T2 reaches conventional
significance individually. T4 is the one tier where the correlation
flattens to $\rho_s \approx 0$. The aggregate effect benefits from
the larger per-model error counts ($\sim 2000$ vs $\sim 400$),
which reduce per-fit noise and let the underlying signal emerge.
The per-tier sign-preservation is the relevant robustness
statement; the aggregate $p$-value is the relevant magnitude
statement.

\begin{table}[h]
\centering
\small
\begin{tabular}{l r r r}
\toprule
tier & $n$ & dense $\rho_s(\log_{10}\text{params}, \sdi_{\text{tier}})$ & $p$ \\
\midrule
T1 (easy)         & 14 & $-0.493$ & $0.073$ \\
T2 (easy)         & 14 & $-0.562$ & $0.037^{*}$ \\
T3 (intermediate) & 14 & $-0.322$ & $0.262$ \\
T4 (hard)         & 14 & $+0.046$ & $0.875$ \\
T5 (hard)         & 14 & $-0.425$ & $0.130$ \\
\midrule
aggregated (all 5) & 14 & $\mathbf{-0.562}$ & $\mathbf{0.006}$ \\
\bottomrule
\end{tabular}
\caption{Per-tier scaling correlation on the $14$ dense models.
Sign preserved in $4/5$ tiers; T2 individually significant. The
addendum's prediction (T4--T5 should drive the headline) is
\emph{not} borne out --- the strongest individual tier is T2.}
\label{tab:per_tier}
\end{table}

\section{\texorpdfstring{S5 --- $m_{\min}$ sensitivity (judge baseline)}{S5 --- m-min sensitivity (judge baseline)}}\label{app:s5}

\Cref{tab:s5} reports each dense model's \sdi{} under three
$m_{\min}$ strategies: (a) our default model-specific KS-selector,
(b) fixed $m_{\min} = 1.5$, and (c) the Clauset-style criterion of
$\geq 100$ exceedances (which selects $m_{\min} = 0.5$ for every
model in the catalog). The dense Spearman correlation between
$\log_{10}(\text{active params})$ and \sdi{} is $-0.562$ under (a),
$+0.837$ under (b), and $+0.793$ under (c). The interpretation of
this sign flip is in \S\ref{sec:sensitivity}.

\begin{table}[h]
\centering
\small
\begin{tabular}{l r r r r}
\toprule
model & \sdi{}$_{\text{def}}$ & \sdi{}$_{m=1.5}$ & \sdi{}$_{m=0.5}$ & $n_{\geq 0.5}$ \\
\midrule
\textsc{llama-3.2-3b}            & 1.046 & 0.469 & 0.290 & 1750 \\
\textsc{phi-3.5-mini}            & 1.309 & 0.530 & 0.298 & 1126 \\
\textsc{gemma-3-4b}              & 0.979 & 0.439 & 0.243 & 1497 \\
\textsc{qwen2.5-7b}              & 1.257 & 0.517 & 0.299 & 1247 \\
\textsc{llama-3.1-8b}            & 1.001 & 0.737 & 0.566 & 2429 \\
\textsc{eurollm-9b}              & 1.067 & 0.526 & 0.297 & 1356 \\
\textsc{solar-10.7b}             & 0.905 & 0.711 & 0.562 & 2571 \\
\textsc{gemma-3-12b}             & 0.938 & 0.742 & 0.555 & 2521 \\
\textsc{ministral-14b}           & 1.122 & 0.795 & 0.534 & 2340 \\
\textsc{mistral-medium-3}        & 0.906 & 0.767 & 0.434 & 2334 \\
\textsc{mistral-small-24b}       & 0.999 & 0.834 & 0.639 & 2265 \\
\textsc{gemma-2-27b}             & 0.619 & 0.786 & 0.619 & 2665 \\
\textsc{gemma-3-27b}             & 0.956 & 0.760 & 0.610 & 2500 \\
\textsc{seed-oss-36b}            & 0.574 & 0.781 & 0.574 & 2260 \\
\midrule
dense Spearman vs.\ $\log_{10}(\text{params})$
                                 & $-0.562$ & $+0.837$ & $+0.793$ &  \\
$p$-value                        & $0.006$  & $0.0002$ & $0.0007$ &  \\
\bottomrule
\end{tabular}
\caption{S5: dense \sdi{} under three $m_{\min}$ strategies. The
default model-specific KS-selector targets the upper tail; fixed
$m_{\min} = 1.5$ and the Clauset-style $\geq 100$-exceedances rule
both target the bulk of the positive distribution. The sign of the
scaling correlation flips between the two regimes: larger dense
models have heavier upper tails but lighter bulk decay
rates.}\label{tab:s5}
\end{table}

\section{Sensitivity figures}\label{app:sensitivity}

\begin{figure}[p]
\centering
\includegraphics[width=0.92\linewidth,height=0.72\textheight,keepaspectratio]{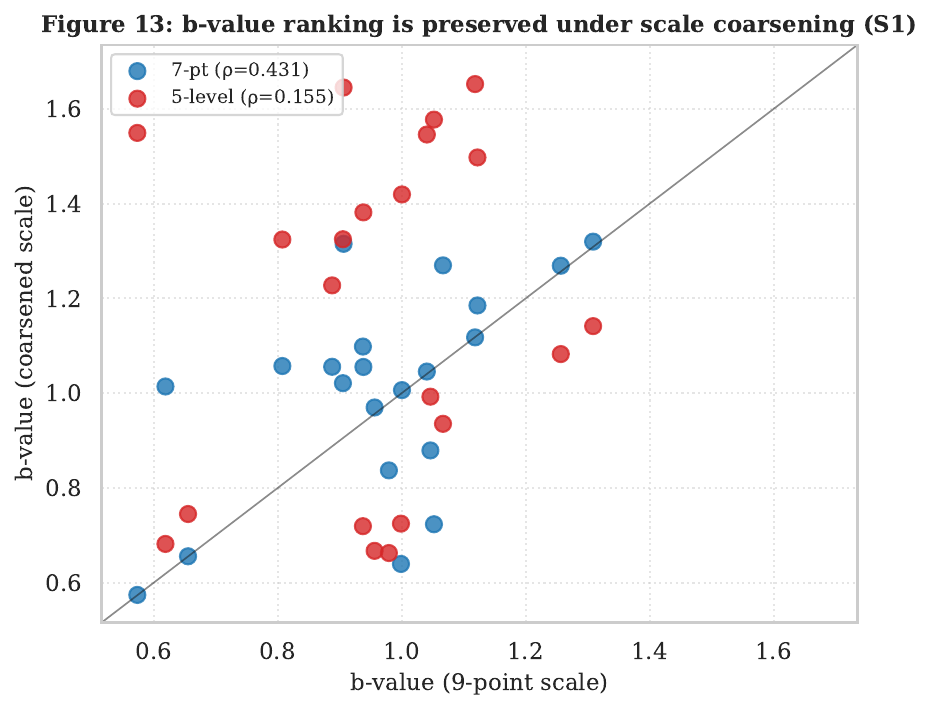}
\caption{S1: \sdi{} under scale coarsening. The 9-point grid is
load-bearing; both coarsenings fall well below the $\rho > 0.85$
stability threshold.}\label{fig:fig13}
\end{figure}

\begin{figure}[p]
\centering
\includegraphics[width=0.92\linewidth,height=0.72\textheight,keepaspectratio]{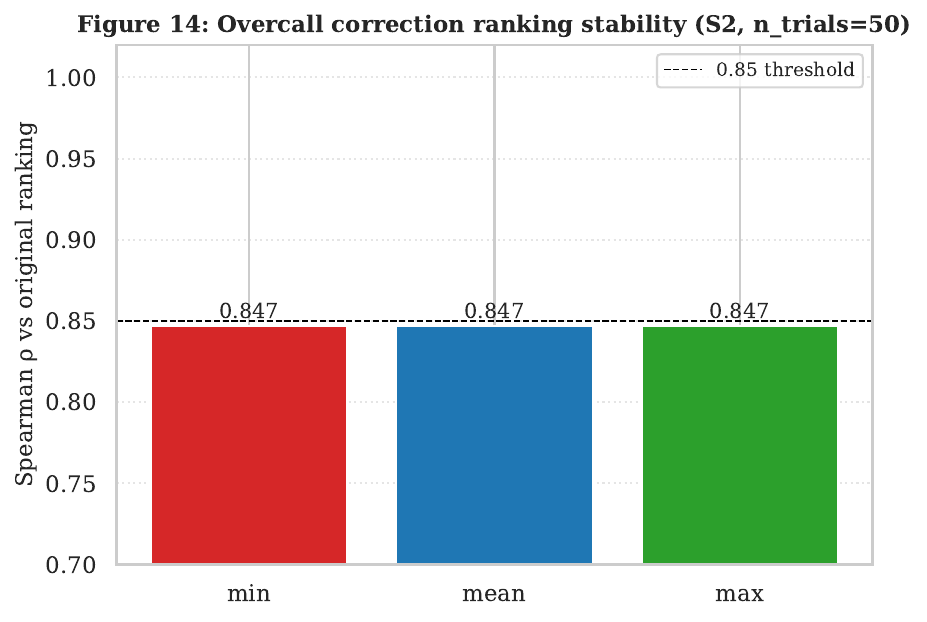}
\caption{S2: ranking stability under overcall correction. Mean
Spearman $\rho = 0.847$ across $50$ bootstrap trials, narrowly
below the pre-registered $0.85$ threshold.}\label{fig:fig14}
\end{figure}

\begin{figure}[p]
\centering
\includegraphics[width=\linewidth,height=0.82\textheight,keepaspectratio]{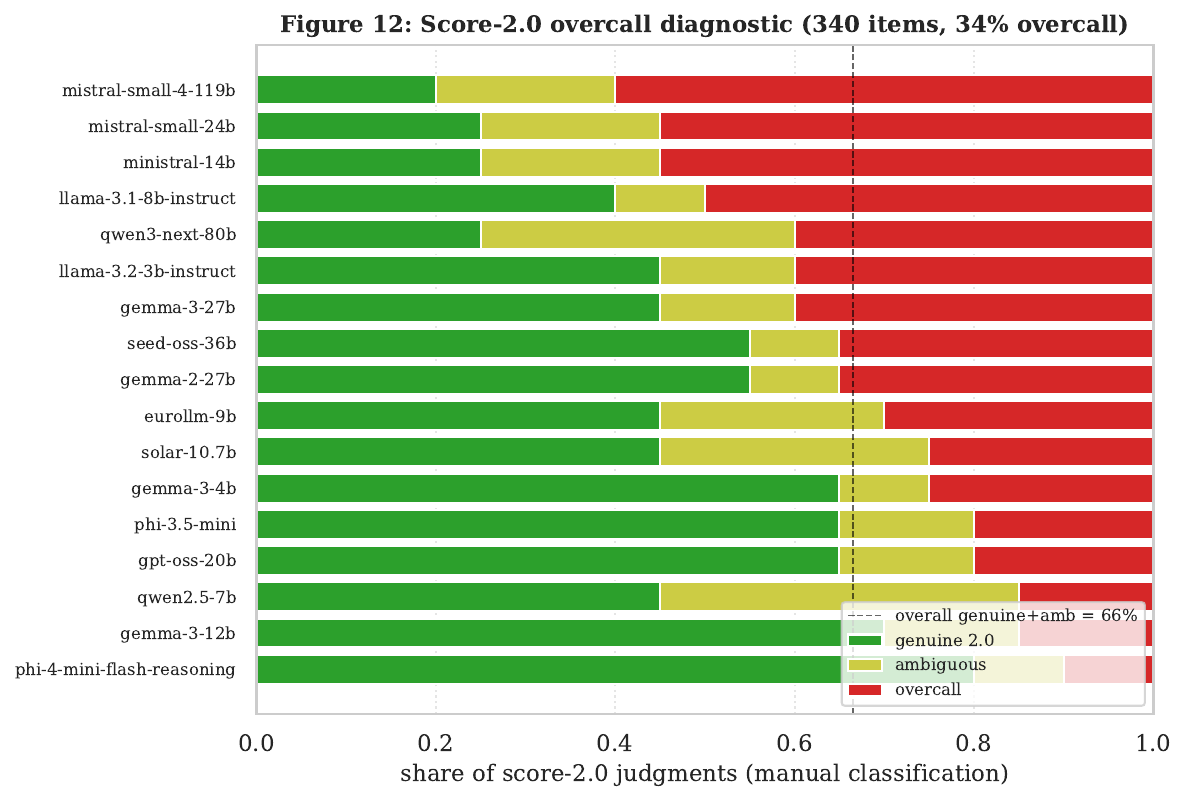}
\caption{Per-model overcall breakdown on the $340$-item
human-rated subset (Appendix~\ref{app:overcall}).}\label{fig:fig12}
\end{figure}

\section*{NeurIPS Paper Checklist}

\begin{enumerate}

\item {\bf Claims}
    \item[] Question: Do the main claims made in the abstract and introduction accurately reflect the paper's contributions and scope?
    \item[] Answer: \answerYes{}
    \item[] Justification: The abstract and introduction summarize the matched-accuracy discriminator result, the theory contribution, the human-validation study, and the scope limits on model coverage. Limitations and scope restrictions are also stated explicitly in the introduction and the Limitations section.

\item {\bf Limitations}
    \item[] Question: Does the paper discuss the limitations of the work performed by the authors?
    \item[] Answer: \answerYes{}
    \item[] Justification: Section~\ref{sec:limitations} discusses judge overcalling, the load-bearing nature of the 9-point scale, model-coverage gaps, limited human-sample size, failed predictive calibration, benchmark misuse, and the fact that queries are LLM-generated.

\item {\bf Theory assumptions and proofs}
    \item[] Question: For each theoretical result, does the paper provide the full set of assumptions and a complete (and correct) proof?
    \item[] Answer: \answerYes{}
    \item[] Justification: Theorem~\ref{thm:nonreduce} and Proposition~\ref{prop:resolution} are stated in Section~\ref{sec:theory}, and complete proofs are provided in Appendix~\ref{app:proofs}. The main text includes proof sketches for intuition and the appendix provides the formal derivations.

\item {\bf Experimental result reproducibility}
    \item[] Question: Does the paper fully disclose all the information needed to reproduce the main experimental results of the paper to the extent that it affects the main claims and/or conclusions of the paper (regardless of whether the code and data are provided or not)?
    \item[] Answer: \answerYes{}
    \item[] Justification: Sections~\ref{sec:method} and \ref{sec:setup}, together with the appendices and released artifact, describe the benchmark construction, model catalog, scoring pipeline, fitting procedure, and statistical analyses used for the main claims. The released artifact also includes scripts and a reproduction guide for rerunning the analyses from saved data.

\item {\bf Open access to data and code}
    \item[] Question: Does the paper provide open access to the data and code, with sufficient instructions to faithfully reproduce the main experimental results, as described in supplemental material?
    \item[] Answer: \answerYes{}
    \item[] Justification: The released artifact includes the code, saved analysis outputs, released data files, Croissant metadata, and reproduction instructions. The paper states that the benchmark, scoring toolkit, human-validation data, and analysis code are released for replication.

\item {\bf Experimental setting/details}
    \item[] Question: Does the paper specify all the training and test details (e.g., data splits, hyperparameters, how they were chosen, type of optimizer) necessary to understand the results?
    \item[] Answer: \answerYes{}
    \item[] Justification: The paper specifies the 10,000-query benchmark design, domain/tier stratification, model catalog, decoding setup, severity scale, dual-judge resolution pipeline, human-validation protocol, candidate distribution families, and bootstrap procedures. Additional tables and prompts are included in the appendix and supplemental files.

\item {\bf Experiment statistical significance}
    \item[] Question: Does the paper report error bars suitably and correctly defined or other appropriate information about the statistical significance of the experiments?
    \item[] Answer: \answerYes{}
    \item[] Justification: The paper reports bootstrap confidence intervals for \sdi{}, p-values for correlation and goodness-of-fit tests, mutual-information decomposition, and additional robustness analyses. The bootstrap setup and significance thresholds are described in the paper and supplemental scripts.

\item {\bf Experiments compute resources}
    \item[] Question: For each experiment, does the paper provide sufficient information on the computer resources (type of compute workers, memory, time of execution) needed to reproduce the experiments?
    \item[] Answer: \answerNo{}
    \item[] Justification: The paper states that inference was run through an open-access third-party API hosting the target models, but it does not provide exact GPU types, memory, or wall-clock runtime for every experiment. We chose not to over-claim compute reproducibility where the provider abstraction hides the underlying hardware.

\item {\bf Code of ethics}
    \item[] Question: Does the research conducted in the paper conform, in every respect, with the NeurIPS Code of Ethics \url{https://neurips.cc/public/EthicsGuidelines}?
    \item[] Answer: \answerYes{}
    \item[] Justification: The paper discusses misuse risks and limitations, and the released artifact is intended for auditing and evaluation rather than unsafe deployment. We are not aware of any aspect of the work that conflicts with the NeurIPS Code of Ethics.

\item {\bf Broader impacts}
    \item[] Question: Does the paper discuss both potential positive societal impacts and negative societal impacts of the work performed?
    \item[] Answer: \answerYes{}
    \item[] Justification: Section~\ref{sec:limitations} includes a Broader impacts paragraph discussing positive impacts for risk auditing and deployment gating, as well as negative impacts from benchmark gaming or optimizing for better-looking but still dangerous failure profiles.

\item {\bf Safeguards}
    \item[] Question: Does the paper describe safeguards that have been put in place for responsible release of data or models that have a high risk for misuse (e.g., pre-trained language models, image generators, or scraped datasets)?
    \item[] Answer: \answerNA{}
    \item[] Justification: The paper releases an evaluation benchmark, scores, and analysis code, not a new generative model or scraped multimedia dataset with elevated release risk. We therefore view the high-risk-model safeguard question as not directly applicable to this artifact.

\item {\bf Licenses for existing assets}
    \item[] Question: Are the creators or original owners of assets (e.g., code, data, models), used in the paper, properly credited and are the license and terms of use explicitly mentioned and properly respected?
    \item[] Answer: \answerNo{}
    \item[] Justification: The paper cites the upstream benchmarks and model families used in the evaluation, and the released artifact includes licenses for the released ERRORQUAKE code/data. However, the paper does not enumerate every upstream model or provider license directly in the manuscript, so we answer conservatively.

\item {\bf New assets}
    \item[] Question: Are new assets introduced in the paper well documented and is the documentation provided alongside the assets?
    \item[] Answer: \answerYes{}
    \item[] Justification: The benchmark release includes a datasheet, Croissant metadata, reproduction instructions, and scripts for the reported analyses. The paper and supplemental material document the query structure, severity scale, scoring fields, and released outputs.

\item {\bf Crowdsourcing and research with human subjects}
    \item[] Question: For crowdsourcing experiments and research with human subjects, does the paper include the full text of instructions given to participants and screenshots, if applicable, as well as details about compensation (if any)? 
    \item[] Answer: \answerNo{}
    \item[] Justification: The released artifact includes the full rating instructions used in the human-validation study. However, the paper does not currently document participant compensation or an equivalent statement that compensation was not applicable.

\item {\bf Institutional review board (IRB) approvals or equivalent for research with human subjects}
    \item[] Question: Does the paper describe potential risks incurred by study participants, whether such risks were disclosed to the subjects, and whether Institutional Review Board (IRB) approvals (or an equivalent approval/review based on the requirements of your country or institution) were obtained?
    \item[] Answer: \answerNo{}
    \item[] Justification: The paper reports a human-rating study but does not include an IRB-approval statement or a fuller discussion of participant-risk disclosure. We therefore answer conservatively rather than implying review documentation that is not present in the released materials.

\item {\bf Declaration of LLM usage}
    \item[] Question: Does the paper describe the usage of LLMs if it is an important, original, or non-standard component of the core methods in this research? Note that if the LLM is used only for writing, editing, or formatting purposes and does \emph{not} impact the core methodology, scientific rigor, or originality of the research, declaration is not required.
    \item[] Answer: \answerYes{}
    \item[] Justification: The paper explicitly states that the benchmark queries were generated by a frontier LLM and that response scoring uses an LLM-judge pipeline with human validation. These components are described in Sections~\ref{sec:method}, \ref{sec:setup}, and \ref{sec:limitations}.

\end{enumerate}

\end{document}